\documentclass[11pt, oneside]{article}   	
\usepackage[a4paper, total={6in, 8in}, margin=1.1in]{geometry}                		

\usepackage{hyperref}
\usepackage{url}

\usepackage{graphicx}
\usepackage{amssymb}
\usepackage{amsmath}
\usepackage{amsfonts}
\usepackage{amsthm}
\usepackage{mathtools}
\usepackage{algorithm}
\usepackage{algcompatible}
\usepackage{thmtools,thm-restate}
\usepackage[affil-it]{authblk}
\usepackage[square,numbers,compress]{natbib}
\usepackage{nicefrac}
\usepackage{comment}
\usepackage{ifpdf}
\usepackage{bm}
\usepackage[table]{xcolor}
\usepackage{color}
\usepackage[displaymath,mathlines]{lineno}
\usepackage{scalerel}
\usepackage{stackengine,wasysym}
\usepackage{hyperref}
\usepackage[stable]{footmisc}
\usepackage{enumerate}
\usepackage{tikz-cd}
\usetikzlibrary{arrows,positioning}
\tikzset{
  shift left/.style ={commutative diagrams/shift left={#1}},
  shift right/.style={commutative diagrams/shift right={#1}}
}
\usepackage{subcaption}
\usepackage{enumitem}
\usepackage{booktabs}
\usepackage{rotating}

\usepackage{hhline} 
\RequirePackage{etoolbox}
\RequirePackage{pgf} 
\RequirePackage{xcolor} 
\definecolor{high}{gray}{0.3}  
\definecolor{low}{gray}{0.9}   

\newtheorem{thm}{Theorem}[section]

\newtheorem{prop}{Proposition}[section]

\newtheorem{dfn}{Definition}

\newtheorem{rk}{Remark}

\newcommand{\vect}[1]{\boldsymbol{#1}}
\newcommand{\tp}[1]{{#1}^{\mathsf T}}

\newcommand{\tr}{\mathrm{tr}}
\newcommand{\VEC}{\mathrm{vec}}

\DeclareMathAlphabet\mathbfcal{OMS}{cmsy}{b}{n}
\DeclareMathOperator{\diag}{diag}
\newcommand{\half}{\frac12}

\newcommand{\E}{\mathrm E}

\newcommand{\mCv}{\mathrm{Cov}}

\newcommand{\KL}{\mathrm{KL}}

\newcommand{\eps}{\epsilon}

\renewcommand{\eps}{\varepsilon}
\renewcommand{\epsilon}{\varepsilon}
\renewcommand{\Sigma}{\varSigma}

\newcommand{\bmu}{\vect\mu}

\newcommand{\bgamma}{{\boldsymbol{\gamma}}}

\newcommand{\bSigma}{\vect\Sigma}

\newcommand{\bV}{{\bf V}}

\newcommand{\bW}{{\bf W}}

\newcommand{\bzero}{{\bf 0}}
\newcommand{\bu}{{\bf u}}
\newcommand{\bU}{{\bf U}}
\newcommand{\by}{{\bf y}}
\newcommand{\bY}{{\bf Y}}

\newcommand{\bx}{{\bf x}}
\newcommand{\bX}{{\bf X}}

\newcommand{\bz}{{\bf z}}
\newcommand{\bZ}{{\bf Z}}

\newcommand{\bM}{{\bf M}}
\newcommand{\bC}{{\bf C}}
\newcommand{\bI}{{\bf I}}
\newcommand{\bL}{{\bf L}}
\newcommand{\bK}{{\bf K}}
\newcommand{\bF}{{\bf F}}
\newcommand{\bS}{{\bf S}}

\newcommand{\mC}{{\mathcal C}}

\newcommand{\mH}{\mathcal{H}}
\newcommand{\mI}{{\mathcal I}}
\newcommand{\mL}{{\mathcal L}}
\newcommand{\mN}{{\mathcal N}}
\newcommand{\mMN}{\mathcal{MN}}

\newcommand{\mQ}{{\mathcal Q}}
\newcommand{\mS}{{\mathcal S}}

\newcommand{\mD}{{\mathcal D}}

\newcommand{\mbR}{{\mathbb R}}

\newcommand{\qED}{\mathrm{q}\!-\!\mathrm{ED}}
\newcommand{\sqED}{\mathrm{q}\!-\!\mathrm{ED}^*}
\newcommand{\qEP}{\mathrm{q}\!-\!\mathcal{EP}}
\newcommand{\ep}{\mathrm{EP}}

\newcommand{\GP}{\mathcal{GP}}

\ifpdf
   \graphicspath{{./figure/PNG/}{./figure/PDF/}{./figure/}}
\else
   \graphicspath{{./figure/EPS/}{./figure/}}
\fi

\title{Deep Q-Exponential Processes}

\author{Zhi Chang \quad Chukwudi Obite \quad Shuang Zhou \quad Shiwei Lan\thanks{slan@asu.edu} \\
School of Mathematical \& Statistical Sciences, Arizona State University \\ 
901 S Palm Walk, Tempe, AZ 85287, USA
}

\begin{document}

\maketitle

\begin{abstract}
Motivated by deep neural networks, the deep Gaussian process (DGP) generalizes the standard GP by stacking multiple layers of GPs. Despite the enhanced expressiveness, GP, as an $L_2$ regularization prior, tends to be over-smooth and sub-optimal for inhomogeneous subjects, such as images with edges. Recently, Q-exponential process (Q-EP) has been proposed as an $L_q$ relaxation to GP and demonstrated with more desirable regularization properties through a parameter $q>0$ with $q=2$ corresponding to GP. Sharing the similar tractability of posterior and predictive distributions with GP, Q-EP can also be stacked to improve its modeling flexibility. In this paper, we generalize Q-EP to deep Q-EP to enjoy both proper regularization and improved expressiveness. The generalization is realized by introducing shallow Q-EP as a latent variable model and then building a hierarchy of the shallow Q-EP layers. Sparse approximation by inducing points and scalable variational strategy are applied to facilitate the inference. We demonstrate the numerical advantages of the proposed deep Q-EP model by comparing with multiple state-of-the-art deep probabilistic models.
\end{abstract}


\section{Introduction}
Gaussian process \citep[GP][]{Rasmussen_2005,Bernardo_1998} has gained enormous successes and been widely used in statistics and machine learning community. With its flexibility in learning functional relationships \citep{Rasmussen_2005} and latent representations \citep{Titsias_2010}, and capability in tractable uncertainty quantification, GP has become one of the most popular non-parametric modeling tools. Facilitated by the sparse approximation \citep{Titsias_2009} and scalable variational inferences \citep[SVGP][]{Hensman_2015, Salimbeni_2017}, GP has been popularized for a variety of big learning tasks. \cite{Neal_1996} in his seminal work discovered that Bayesian neural networks with infinite width converged to GP with certain kernel function. Inspired by the advancement of deep learning \citep{Goodfellow-et-al-2016}, \cite{Damianou_2013} pioneered in generalizing GP with deep structures, hence named deep GP. Ever since then, there has been a large volume of follow-up works including 
deep convolutional GP \citep{Blomqvist_2020},
deep sigma point process \citep[DSPP][]{Jankowiak_2020b}, deep image prior \citep{Ulyanov2020}, deep kernel process \citep{Aitchison_2021}, deep variational implicit process \citep{ortega_2023}, deep horseshoe GP \citep{Castillo_2024}, and various applications \citep{Dutordoir_2020,Li_2021,Jones_2023}.

Despite its flexibility, GP, as an $L_2$ regularization method, tends to produce random candidate functions that are over-smooth and thus sub-optimal for modeling inhomogeneous objects with abrupt changes or sharp contrast. Recently, an $L_q$ based stochastic process, $Q$-exponential process \citep[Q-EP][]{Li_2023}, has been proposed to impose flexible regularization through a parameter $q>0$, which includes GP as a special case when $q=2$. Different from other $L_1$ based priors such as Laplace random field \citep{Podg_rski_2011, KOZUBOWSKI_2013} and Besov process \citep{Lassas_2009, Dashti_2012}, Q-EP shares with GP the unique tractability of posterior and predictive distributions \citep[Theorem 3.5 of][]{Li_2023}, which essentially permits a deep generalization by stacking multiple associated stochastic mappings \citep{Damianou_2013}.

Motivated by the enhanced expressiveness of deep GP and the flexible regularization of Q-EP, in this work we generalize Q-EP to \emph{deep Q-EP} to enjoy both merits. On one hand, by stacking multiple layers of Q-EP mappings, deep Q-EP becomes more capable of characterizing complex latent representations than the standard Q-EP. On the other hand, inherited from Q-EP, deep Q-EP maintains the control of regularization through the parameter $q>0$, whose smaller values impose stronger regularization, more amenable than (deep) GP to preserve inhomogenous traits such as edges in an image.
First, we introduce the building block, shallow Q-EP model, which can be regarded as a kernelized latent variable model \citep{Lawrence_2003, Titsias_2010}.
Such shallow model is also viewed as a stochastic mapping $F$ from input (or latent) variables $X$ to output variables $Y$ defined by a kernel. Then as in \cite{Lawrence_2007, Damianou_2013}, we extend such mapping by stacking multiple shallow Q-EP layers to form a hierarchy for the deep Q-EP. Sparse approximation by inducing points \citep{Titsias_2009} is adopted for the variational inference of deep Q-EP. Unlike the original deep GP \citep{Damianou_2013} relying variational calculus to calculate the variational distribution for the function values on inducing points, we adopt the two-stage strategy in \cite{Hensman_2015} that computes the variational distribution of $F$, which is more succinct and scalable. Though not straightforward in the setting of Q-EP, tractable evidence lower bound (ELBO) for the logarithm of model evidence can be obtained with the help of Jensen's inequality. The inference procedure, as in deep GP, can be efficiently implemented in \texttt{GPyTorch} \citep{Gardner_2018}.

\paragraph{Connection to existing works} 
Our proposed deep Q-EP is closely related to deep GP \citep{Damianou_2013} and two other works, deep kernel learning \citep[DKL-GP][]{Wilson_2016} and DSPP \citep{Jankowiak_2020b}. Deep Q-EP generalizes deep GP with a parameter $q>0$ to control the regularization (See Figure \ref{fig:LVM} for its effect on learning representations) and includes deep GP as a special case for $q=2$. DKL-GP combines the deep learning architectures (neural networks) with the non-parametric flexibility of kernel methods (GP). The GP part can also be replaced by Q-EP to evolve to new methods like DKL-QEP (See Section \ref{sec:imgcls}.) DSPP is motivated by parametric GP models \citep[PPGPR][]{Jankowiak_2020a} and applies sigma point approximation or quadrature-like integration to the predictive distribution. The majority of popular deep probabilistic models rely on GP. This is one of the few developed out of a non-Gaussian stochastic process.
Our proposed work on deep Q-EP has multi-fold contributions to deep probabilistic models:
\begin{enumerate}
    \item We propose a novel deep probabilistic model based on Q-EP that generalizes deep GP with flexibility of regularization.
    \item We develop the variational inference for deep Q-EP and efficiently implement it.
    \item We demonstrate numerical advantages of deep Q-EP in handling data inhomogeneity by comparing with state-of-the-art deep probabilistic models.
\end{enumerate}

The rest of the paper is organized as follows. Section \ref{sec:QEP} introduces the background of Q-EP. We then develop shallow Q-EP in Section \ref{sec:QEP-LVM} as the building block for deep Q-EP in Section \ref{sec:deep-QEP}. In these two sections, we highlight the importance of posterior tractability in the development and some obstacles in deriving the variational lower bounds. In Section \ref{sec:numerics} we demonstrate the numerical advantages by comparing with multiple deep probabilistic models in various learning tasks. Finally, we conclude with some discussion on the limitation and potential improvement in Section \ref{sec:conclusion}.

\section{Background: $Q$-exponential Processes}\label{sec:QEP}
\subsection{Multivariate $Q$-exponential Distribution}
Based on $L_q$ regularization, the univariate \emph{$q$-exponential distribution} \citep{Dashti_2012} with an inexact density (not normalized to 1), $\pi_q(u) \propto \exp{(- \half|u|^q)}$, 
is one of the following exponential power (EP) distributions $\ep(\mu, \sigma, q)$ with $\mu=0$, $\sigma=1$:
\begin{equation*}\label{eq:epd}
    p(u|\mu, \sigma, q) = \frac{q}{2^{1+1/q}\sigma\Gamma(1/q)}\exp\left\{-\half \left|\frac{u-\mu}{\sigma}\right|^q\right\} .
\end{equation*}
This family includes normal distribution $\mN(\mu,\sigma^2)$ for $q=2$ and Laplace distribution $L(\mu, b)$ with $\sigma=2^{-1/q} b$ for $q=1$ as special cases.

\cite{Li_2023} generalizes the univariate $q$-exponential random variable to a multivariate random vector on which a stochastic process can be defined with two requirements by the Kolmogorov' extension theorem \citep{Oksendal_2003}: i) {\bf exchangeability} of the joint distribution, i.e. $p(\bu_{1:N}) = p(\bu_{\tau(1:N)})$ for any finite permutation $\tau$; and ii) {\bf consistency} of marginalization, i.e. $p(\bu_1) = \int p(\bu_1, \bu_2)d\bu_2$.

Suppose a function $u(x)$ is observed at $N$ locations, $x_1,\cdots,x_N\in \mD\subset\mbR^{d}$. 
\cite{Li_2023} find a consistent generalization, named \emph{multivariate $q$-exponential distribution},
for $\bu=(u(x_1),\cdots, u(x_N))$ from the family of elliptic contour distributions \citep{Johnson_1987,fang1990generalized}.
\begin{dfn}\label{dfn:qED}
    A multivariate $q$-exponential distribution for a random vector $\bu\in\mbR^N$, denoted as $\qED_N(\bmu, \bC)$, has the following density
    \begin{equation}\label{eq:qED}
        p(\bu|\bmu, \bC, q) = \frac{q}{2} (2\pi)^{-\frac{N}{2}} |\bC|^{-\half} r^{(\frac{q}{2}-1)\frac{N}{2}} \exp\left\{-\frac{r^\frac{q}{2}}{2}\right\}, \quad r(\bu) = \tp{(\bu-\bmu)} \bC^{-1} (\bu-\bmu).
    \end{equation}
\end{dfn}
The following proposition describes the role of matrix $\bC$ in characterizing the covariance between the components \citep{Li_2023}.
\begin{prop}\label{prop:qED_cov}
    If $\bu\sim \qED_N(\bmu, \bC)$, then we have
    \begin{equation*}
        \E[\bu] = \bmu, \qquad \mCv(\bu) = \frac{2^{\frac{2}{q}}\Gamma(\frac{N}{2}+\frac{2}{q})}{N\Gamma(\frac{N}{2})} \bC \overset{\cdot}{\sim} N^{\frac{2}{q}-1} \bC, \quad as\quad N\to \infty.
    \end{equation*}
\end{prop}

\subsection{$Q$-exponential Process and Multi-output Q-EP}
\cite{Li_2023} prove that the multivariate $q$-exponential random vector $\bu\sim \qED_N(0, \bC)$ 
satisfies the conditions of Kolmogorov's extension theorem hence it can be generalized to a stochastic process.
For this purpose, we scale it by a factor $N^{\half-\frac{1}{q}}$ so that its covariance is asymptotically finite (refer to Proposition \ref{prop:qED_cov}). If $\bu\sim \qED_N(0, \bC)$, then we denote $\bu^*:=N^{\half-\frac{1}{q}}\bu \sim \sqED_N(0, \bC)$ as a \emph{scaled} $q$-exponential random variable.
With a covariance (symmetric and positive-definite) kernel $\mC : \mD\times \mD\to \mbR$,
we define the following \emph{$q$-exponential process (Q-EP)} based on the scaled $q$-exponential distribution $\sqED_N(0, \bC)$.
\begin{dfn}\label{dfn:qEP}
    A (centered) $q$-exponential process $u(x)$ with kernel $\mC$, $\qEP(0, \mC)$, is a collection of random variables such that any finite set, $\bu:=(u(x_1),\cdots u(x_N))$, follows a scaled multivariate $q$-exponential distribution $\sqED(0, \bC)$, where $\bC=[\mC(x_i,x_j)]_{N\times N}$.
    If $\mC=\mI$, then $u$ is said to be \emph{marginally identical but uncorrelated (m.i.u.)}.
\end{dfn}
\begin{rk}
When $q=2$, $\qED_N(\bmu, \bC)$ reduces to $\mN_N(\bmu, \bC)$ and $\qEP(0, \mC)$ becomes $\GP(0,\mC)$. When $q\in[1,2)$, $\qEP(0, \mC)$ lends flexibility to modeling functional data with more regularization than GP.
\end{rk}

One caveat of Q-EP is that uncorrelation (identity covariance) does not imply independence except for the special Gaussian case ($q=2$). For multiple Q-EPs, $(u_1(x), \cdots, u_D(x))$, we usually do not assume them independent because their joint distribution is difficult to work with (due to the lack of additivity in the exponential part of density function \eqref{eq:qED}). Rather, uncorrelation is a preferable assumption. In general, we define multi-output (multivariate) Q-EPs through matrix vectorization.
\begin{dfn}\label{dfn:multi_qEP}
    A multi-output (multivariate) \emph{$q$-exponential process}, $u(\cdot)=(u_1(\cdot), \cdots, u_D(\cdot))$, each $u_j(\cdot)\sim\qEP(\mu_j,\mC_x)$, is said to have association $\bC_t$ if at any finite locations, $\bx=\{x_n\}_{n=1}^N$, $\VEC([u_1(\bx),\cdots,u_D(\bx)]_{N\times D})\sim \qED_{ND}(\VEC(\bmu), \bC_t\otimes \bC_x)$, where we have $u_j(\bx)=\tp{[u_j(x_1),\cdots,u_j(x_N)]}$, for $j=1, \ldots, D$, $\bmu=[\mu_1(\bx),\cdots,\mu_D(\bx)]_{N\times D}$ and $\bC_x=[\mC_x(x_n,x_m)]_{N\times N}$. We denote $u\sim \qEP(\mu,\mC_x, \bC_t)$.
    In particular, $\{u_j(\cdot)\}$ are m.i.u. if $\bC_t=\bI_D$.
\end{dfn}

In the following, we will stack m.i.u. multi-output Q-EPs to build a deep Q-EP.

\subsection{Bayesian Regression with Q-EP Priors}
Given data $\bx=\{x_n\}_{n=1}^N$ and $\by=\{y_n\}_{n=1}^N$, we consider the generic Bayesian regression model:
\begin{equation}\label{eq:regression}
\begin{aligned}
\by &= f(\bx) + \vect\eps, \quad \vect\eps \sim \qED_N(0,\Sigma),\\
f &\sim \qEP(0,\mC) .
\end{aligned}
\end{equation}
\cite[Theorem 3.5 of][]{Li_2023} states the tractable posterior (predictive) distribution, which is the key point of the deep generalization.
\begin{thm}
For the regression model \eqref{eq:regression}, the posterior distribution of $f(x_*)$ at $x_*$ is
\begin{equation*}\label{eq:post_pred}
f(x_*)|\by,\bx,x_*  \sim \qED(\bmu^*, \bC^*), \quad \bmu^*=\tp{\bC}_*(\bC+\Sigma)^{-1}\by, \quad \bC^*=\bC_{**}-\tp{\bC}_*(\bC+\Sigma)^{-1}\bC_* ,
\end{equation*}
where $\bC=\mC(\bx,\bx)$, $\bC_*=\mC(\bx,x_*)$, and $\bC_{**}=\mC(x_*,x_*)$.
\end{thm}

Denote $\bX=[\bx_1,\cdots, \bx_Q]_{N\times Q}$, $\bF=[f_1(\bX),\cdots,f_D(\bX)]_{N\times D}$ and $\bY=[\by_1,\cdots, \by_D]_{N\times D}$.
With m.i.u. Q-EP priors as in Definition \eqref{dfn:multi_qEP} imposed on $f:=(f_1,\cdots,f_D)$, we now consider the following multivariate regression problem:
\begin{equation}\label{eq:multi_regression}
\begin{aligned}
\textrm{likelihood}:\quad    \VEC(\bY) | \bF &\sim \qED_{ND}(\VEC(\bF), \bI_D\otimes\Sigma), \\
\textrm{prior on latent function}:\quad    
f &\sim \qEP(0, \mC, \bI_D).
\end{aligned}
\end{equation}
Based on the additivity of $\qED$ (as a special elliptic contour) random variables \citep{fang1990generalized}, we can find the marginal of $\bY$ by
noticing that $\bY = \bF + \vect\eps$ with $\VEC(\vect\eps)\sim \qED(\bzero, \bI_D\otimes\Sigma)$:
\begin{equation}\label{eq:marglik}
\textrm{marginal likelihood}:\quad    \VEC(\bY)|\bX \sim \qED_{ND}(\bzero, \bI_D\otimes(\bC+\Sigma)).
\end{equation}

\section{Shallow Q-EP Model}\label{sec:QEP-LVM}
In this section we introduce a shallow (1-layer) Q-EP model which serves as a building block for the deep Q-EP model to be developed in Section \ref{sec:deep-QEP}. We start with the the marginal model \eqref{eq:marglik} that can be identified as a latent variable model \citep{Lawrence_2003} with specified kernel. This defines a shallow Q-EP model. Then we develop variational infererence with sparse approximation for such model \citep{Titsias_2010} and stack multiple layers to form a deep Q-EP.

Note the marginal model \eqref{eq:marglik} of $\bY|\bX$ can be viewed as a stochastic mapping \citep[Theorem 2.1 of][]{Li_2023}:
\begin{equation*}\label{eq:stoch_map}
    \tilde f: \bX \to \bY = R \bL \bS ,
\end{equation*}
where $R^q\sim \chi^2(N)$, $\bL$ is the Cholesky factor of $\bC_\bX+\Sigma$ whose value depends on $\bX$, and $\bS:=[S_1,\cdots, S_D]\sim \mathrm{Unif}(\prod_{d=1}^D \mS^{N+1})$, i.e. each $S_d$ is uniformly distributed on an $N$-dimensional unit-sphere $\mS^{N+1}$.

Note $\bX$ is an input variable in the supervised learning, and could also be a latent variable in the unsupervised learning.
In the latter case, the shallow Q-EP model \eqref{eq:marglik} of $\bY|\bX$ can be regarded a latent variable model obtained by integrating out the latent function $\bF$ in model \eqref{eq:multi_regression}, which is a linear mapping in probabilistic PCA \citep{Tipping_1999} and a multi-output GP in GP-LVM \citep{Lawrence_2003,Lawrence_2005}. GP can be replaced by Q-EP to impose flexible regularization on the input (latent) space, and hence we propose the shallow Q-EP model as also a Q-EP LVM.

For the convenience of exposition, we set $\Sigma=\beta^{-1}\bI_N$ and denote $\bK:=\bC_\bX+\Sigma$. We adopt the following 
automatic relevance determination (ARD) kernel as in \cite{Titsias_2010}, e.g. squared exponential (SE), to determine the dominant dimensions in the input (latent) space:
\begin{equation}\label{eq:ARD}
    \bK=[k(\bx_n, \bx_m)]_{N\times N}, \quad 
    k(\bx_n, \bx_m) = \alpha^{-1} \exp\left\{-\half \tp{(\bx_n-\bx_m)} \diag(\bgamma) (\bx_n-\bx_m) \right\} .
\end{equation}

%

 The probabilistic PCA can be reformulated as an LVM by integrating out the linear mapping $\bF=\bX\bW$ (through the parameters $\bW$) \citep{Lawrence_2003}. In the Q-EP setting, the resulted model becomes \eqref{eq:marglik} with $\bK=\alpha^{-1}\bX\tp{\bX}+\beta^{-1}\bI_N$.
 Denote $r(\bY)=\tp{\VEC(\bY)} (\bI_D\otimes\bK)^{-1}\VEC(\bY)= \tr(\bK^{-1}\bY\tp{\bY})$.
 Then the log-likelihood is
 \begin{equation}\label{eq:log-likelihood}
     L = -\frac{D}{2}\log|\bK| + \frac{ND}{2}\left(\frac{q}{2}-1\right) \log r(\bY) - \half r^{\frac{q}{2}}(\bY).
 \end{equation}
 The following theorem states that the maximum likelihood estimator (MLE) for $\bX$ is equivalent to the solution for this probabilistic PCA \citep{Tipping_1999, Minka_2000} with Q-EP prior.
 \begin{thm}\label{thm:prob_PCA}
 Suppose $\bY\tp{\bY}$ has eigen-decomposition $\bU \vect\Lambda \tp{\bU}$ with $\vect\Lambda$ being the diagonal matrix with eigenvalues $\{\lambda_i\}_{i=1}^N$. Then the MLE for \eqref{eq:log-likelihood} is
 \begin{equation*}
 \bX=\bU_Q \bL \bV, \quad \bL = \diag(\{\sqrt{\alpha(c\lambda_i-\beta^{-1})}\}_{i=1}^Q), \quad c=\left[\frac{q(D\wedge Q)^{\frac{q}{2}}}{2D(D\wedge Q) + (q-2) ND}\right]^{\frac{2}{q}},
 \end{equation*}
 where $\bU_Q$ is an $N\times Q$ matrix with the first $Q$ eigen-vectors in $\bU$, and $\bV$ is an arbitrary $Q\times Q$ orthogonal matrix.
 \end{thm}
 \begin{proof}
 See Appendix \ref{apx:prob_PCA}.
 \end{proof}
 \begin{rk}
 When $q=2$, the eigenvalues reduce to $l_i=\sqrt{\alpha(D^{-1}\lambda_i-\beta^{-1})}$ corresponding to those for GP-LVM, though they were mistakenly stated as $l_i=(\alpha^{-1}(D^{-1}\lambda_i-\beta^{-1}))^{-\frac{1}{2}}$ in \cite{Lawrence_2003}.
 \end{rk}

\subsection{Bayesian Shallow Q-EP}\label{sec:Bayes_LVM}
Like \cite{Titsias_2010}, we adopt a prior for the input (latent) variable $\bX$ and introduce the following Bayesian shallow Q-EP model:
\begin{equation}\label{eq:Bayes_LVM}
\begin{aligned}
\textrm{marginal likelihood}:\quad    \VEC(\bY) | \bX &\sim \qED(\bzero, \bI_D\otimes\bK), \\
\textrm{prior on input/latent variable}:\quad    \VEC(\bX) &\sim \qED(\bzero, \bI_{QD}).
\end{aligned}
\end{equation}
Compared with the optimization method \citep{Lawrence_2003}, the Bayesian training procedure \citep{Titsias_2010} is robust to overfitting and can automatically determine the intrinsic dimensionality of the nonlinear input (latent) space.

To scale up the application, we consider variational Bayes to train the shallow Q-EP model \eqref{eq:Bayes_LVM}.
The variational inference for this model is much more complicated than GP because the log-likelihood \eqref{eq:multi_regression} is no longer represented as a quadratic form of data. Yet an appropriate evidence lower bound (ELBO) can still be obtained with the help of Jensen's inequality.

For variational Bayes, we approximate the posterior distribution $p(\bX|\bY)\propto p(\bY|\bX)p(\bX)$ with the uncorrelated $\qED$: 
\begin{equation*}\label{eq:var_dist}
    q(\bX) \sim \qED(\bmu, \diag(\{\bS_n\})),
\end{equation*}
where each covariance $\bS_n$ is of size $D\times D$ and can be chosen as a diagonal matrix for convenience. 

To speed up the computation, sparse variational approximation \citep{Titsias_2009, Lawrence_2007} is adopted by introducing the inducing points $\tilde\bX\in\mbR^{M\times Q}$ with their function values $\bU=[f_1(\tilde\bX),\cdots,f_D(\tilde\bX)]\in\mbR^{M\times D}$.
Hence the marginal likelihood $p(\bY|\bX)$ in \eqref{eq:Bayes_LVM} can be augmented to a joint distribution of several $\qED$ random variables:
\begin{equation*}
    p(\bY|\bX)\propto p(\bY|\bF) p(\bF|\bU,\bX,\tilde\bX) p(\bU|\tilde\bX) ,
\end{equation*}
where $p(\VEC(\bF)|\bU,\bX,\tilde\bX) \sim \qED(\VEC(\bK_{NM}\bK_{MM}^{-1}\bU), \bI_D\otimes(\bK_{NN}-\bK_{NM}\bK_{MM}^{-1}\bK_{MN}))$ and $p(\VEC(\bU)|\tilde\bX) \sim \qED(\bzero, \bI_D \otimes\bK_{MM})$.

Denote by $\varphi(r; \Sigma, D):=-\frac{D}{2}\log |\Sigma| + \frac{ND}{2}\left(\frac{q}{2}-1\right) \log r - \half r^{\frac{q}{2}}$. 
With the variational distribution $q(\bF,\bU,\bX)=p(\bF|\bU,\bX)q(\bU)q(\bX)$ for $q(\bU) \sim \qED(\bM, \diag(\{\bSigma_d\}))$, the following final ELBO is obtained by the two-stage approach in \cite[SVGP][]{Hensman_2015} (Refer to Section \ref{sec:elbo_qeplvm} for details):
\begin{equation}\label{eq:LVM_ELBO}
\begin{aligned}
    \log p(\bY) \geq & \mL(q) = \int q(\bX)q(\bU)p(\bF|\bU,\bX)\log\frac{p(\bY|\bF)p(\bU)p(\bX)}{q(\bU)q(\bX)} d\bF d\bU d\bX \\
    \geq & h^*(\bY, \bX) - \KL^*_\bU - \KL^*_\bX, \\
    h^*(\bY, \bX) =& \varphi(r_\bY; \beta^{-1}\bI_N, D), \\
    r_\bY =& r(\bY, \Psi_1\bK_{MM}^{-1}\bM) + \beta\tr( \tp{\bM}\bK_{MM}^{-1}(\Psi_2-\tp{\Psi}_1\Psi_1)\bK_{MM}^{-1}\bM) \\
    &+ \beta D[\psi_0-\tr(\bK_{MM}^{-1}\Psi_2)] + \beta\sum_{d=1}^D \tr(\bK_{MM}^{-1}\bSigma_d \bK_{MM}^{-1}\Psi_2), \\
    - \KL^*_\bU =& \half \sum_{d=1}^D\log |\bSigma_d| 
    + \varphi\left(\tr(\tp{\bM}\bK_{MM}^{-1}\bM) + \sum_{d=1}^D\tr(\bSigma_d\bK_{MM}^{-1}); \bK_{MM}, D\right), \\
    - \KL^*_\bX =& \half \sum_{n=1}^N\log |\bS_n| 
    + \varphi\left(\tr(\tp{\bmu}\bmu) + \sum_{n=1}^N\tr(\bS_n); \bI_N, Q\right),
\end{aligned}
\end{equation}
where $\psi_0=\tr(\langle\bK_{NN}\rangle_{q(\bX)})$, $\Psi_1= \langle\bK_{NM}\rangle_{q(\bX)}$, and $\Psi_2 = \langle\bK_{MN}\bK_{NM}\rangle_{q(\bX)}$. 

\begin{rk}
When $q=2$, $\varphi(r; \Sigma, D)=-\frac{D}{2}\log |\Sigma| - \half r$ with $r=r(\bY, \Psi_1\bK_{MM}^{-1}\bM)$ becomes the log-density of matrix normal $\mMN_{N\times D}(\Psi_1\bK_{MM}^{-1}\bM, \beta^{-1}\bI_N, \bI_D)$. Then the ELBO \eqref{eq:LVM_ELBO} reduces to the ELBO as in Equation (7) of \cite[SVGP][]{Hensman_2015} with an extra term $\beta\tr( \tp{\bM}\bK_{MM}^{-1}(\Psi_2-\tp{\Psi}_1\Psi_1)\bK_{MM}^{-1}\bM)$.
\end{rk}

\begin{figure}[t]
  \centering
  \includegraphics[width=1\textwidth,height=.2\textwidth]{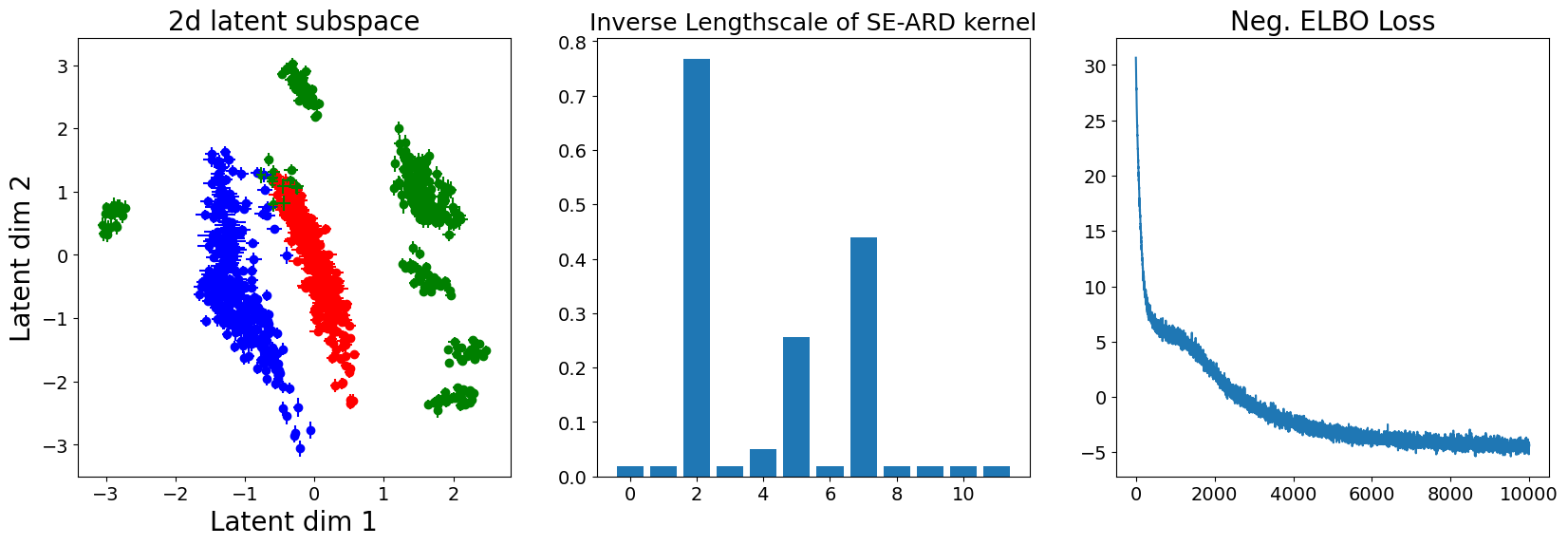}
  \includegraphics[width=1\textwidth,height=.2\textwidth]{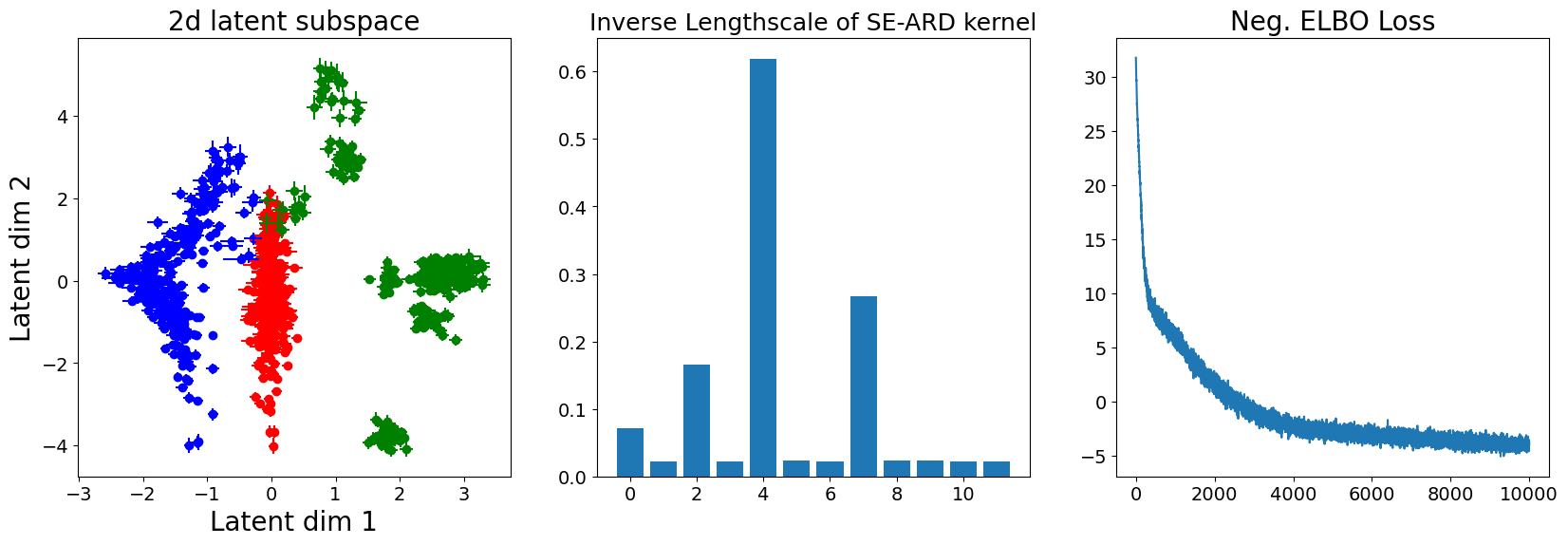}
  \includegraphics[width=1\textwidth,height=.2\textwidth]{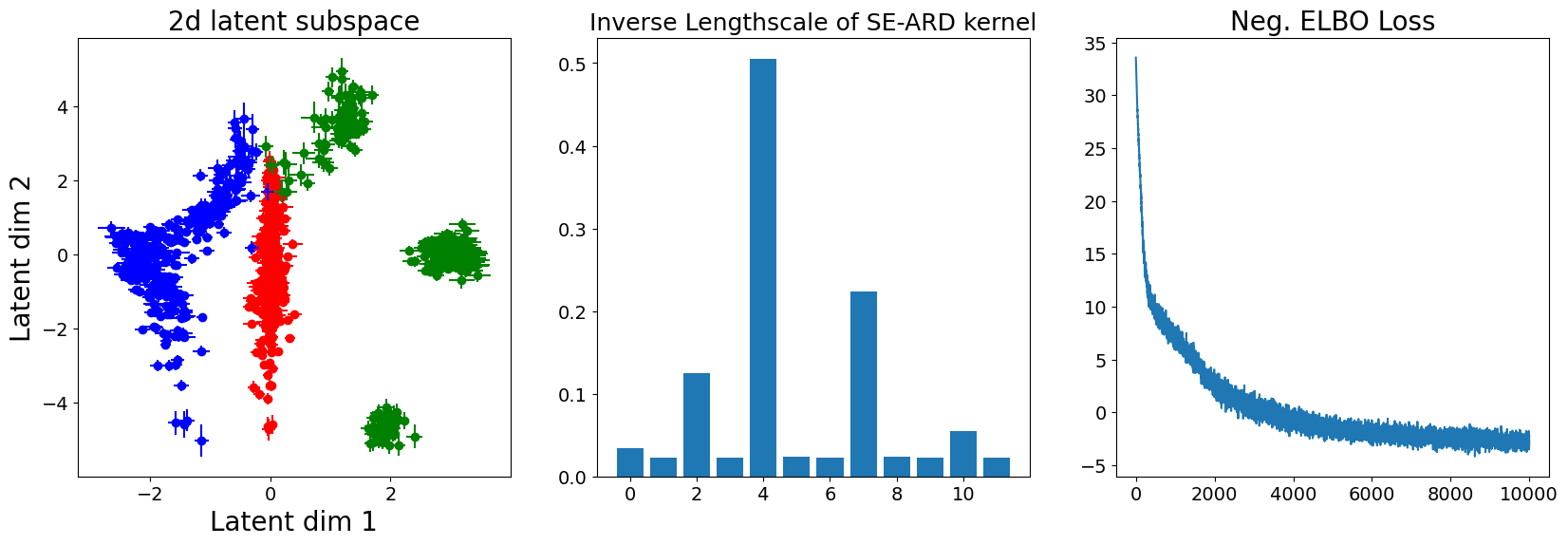}
  \caption{2d latent space of multi-phase oil-flow dataset: contrasting GP-LVM (top row) with two shallow Q-EPs for $q=1.25$ (middle row) and $q=1$ (bottom row).}
  \label{fig:LVM}
\end{figure}

We demonstrate the behavior of shallow Q-EP as an LVM in unsupervised learning and contrast it with GP-LVM using the canonical multi-phase oil-flow dataset used in \cite{Titsias_2010} that consists of 1000 observations (12-dimensional) corresponding to three different phases of oil-flow.
Figure \ref{fig:LVM} visualizes 2d slices of the latent spaces identified with the most dominant latent dimensions found by GP-LVM (top) and two shallow Q-EP models with $q=1.25$ (middle) and $q=1$ (bottom) respectively. The vertical and horizontal bars indicate axis aligned uncertainty around each latent point.
As GP-LVM corresponds to a shallow Q-EP with $q=2$, the parameter $q>0$ controls a regularization effect of shallow Q-EP: the smaller $q$ leads to more regularization on the learned latent representations and hence yields more aggregated clusters, as illustrated by the green class in the first column of Figure \ref{fig:LVM}.
The two types of models also differ in the dominant relevant dimensions: $(2,5,7)$ for GP-LVM versus $(2,4,7)$ for QEP-LVM. Note, the ELBO loss of shallow Q-EP converges slightly faster than that of GP-LVM in this example, yet their final values are not comparable because two models have different densities.

\section{Deep Q-EP Model}\label{sec:deep-QEP}
In this section, we construct the deep Q-EP model by stacking multiple shallow Q-EP layers introduced in Section \ref{sec:QEP-LVM}, similarly as building deep GP with GP-LVMs \citep{Damianou_2013}.
More specifically, we consider a hierarchy of $L$ shallow Q-EP layers \eqref{eq:Bayes_LVM} as follows:
\begin{equation*}\label{eq:LVM_stacks}
\begin{aligned}
    y_{nd} &= f_d^0(\bx^1_n) + \eps_{nd}^0, \quad d=1,\cdots, D_0, \quad \bx^1_n\in \mbR^{D_1}, \\
    x^1_{nd} &= f_d^1(\bx^2_n) + \eps_{nd}^1, \quad d=1,\cdots, D_1, \quad \bx^2_n\in \mbR^{D_2}, \\
    \vdots &\phantom{= f_d^1(\bx^2_n) +} \vdots \phantom{\eps_{nd}^1, \quad d=1,} \vdots \phantom{\cdots, D_1, \quad \bx^2_n} \vdots\\
    x^{L-1}_{nd} &= f_d^{L-1}(\bz_n) + \eps_{nd}^{L-1}, \quad d=1,\cdots, D_{L-1}, \quad \bz_n\in \mbR^{D_L},
\end{aligned}
\end{equation*}
where $\vect{\eps}^\ell\sim \qED(\bzero, \Gamma^\ell)$, $f^\ell\sim \qEP(0, k^\ell, I_{D_\ell})$ for $\ell=0,\cdots,L-1$ and we identify $\bY=\bX^0$ and $\bZ=\bX^{L}$.

Consider the prior $\bZ\sim \qED(\bzero, \bI_{ND_L})$. The joint probability, augmented with the inducing points $\tilde\bX^\ell$ and the associated function values $\bU^\ell=[f_d^\ell(\tilde\bX^\ell)]_{d=1}^{D_\ell}$, is decomposed as
\begin{equation*}
    p(\{\bX^\ell,\bF^\ell,\bU^\ell\}_{\ell=0}^{L-1},\bZ) = \prod_{\ell=0}^{L-1} p(\bX^\ell|\bF^\ell) p(\bF^\ell|\bU^\ell, \bX^{\ell+1}) p(\bU^\ell) \cdot p(\bZ).
\end{equation*}
And we use the following variational distribution
\begin{equation*}
    \mQ=\prod_{\ell=0}^{L-1} p(\bF^\ell|\bU^\ell, \bX^{\ell+1}) q(\bU^\ell) q(\bX^{\ell+1}), \quad q(\bX^{\ell+1}) = \qED(\bmu^{\ell+1}, \diag(\{\bS^{\ell+1}_n\})).
\end{equation*}
Then the ELBO becomes
\begin{equation*}
\begin{aligned}
    \mL(\mQ) &= \int_{\{\bF^\ell,\bU^\ell,\bX^{\ell+1}\}_{\ell=0}^{L-1}} \mQ\log \frac{p(\{\bX^\ell,\bF^\ell,\bU^\ell\}_{\ell=0}^{L-1},\bZ)}{\prod_{\ell=0}^{L-1} q(\bU^\ell) q(\bX^{\ell+1})} \\
    &= h_0 - \KL_{\bU^0} + \sum_{\ell=1}^{L-1}[h_\ell -\KL_{\bU^\ell} +\mH_q(\bX_\ell)] - \KL_\bZ,
\end{aligned}
\end{equation*}
where $h_\ell= \left\langle \log p(\bX^\ell|\bF^\ell) \right\rangle_{q(\bF^\ell)q(\bX^{\ell+1})q(\bX^\ell)}$ with $q(\bX^0)=q(\bY)\equiv 1$.
Based on the previous bound \eqref{eq:LVM_ELBO}, we have for $\ell=1,\cdots,L-1$ (Refer to Section \ref{sec:elbo_deepqep} for details):
\begin{equation*}
\begin{aligned}
    h_0 \geq&  h^*(\bY, \bX^1),  \\
    h_\ell \geq& h^*(\bX^\ell, \bX^{\ell+1})= \varphi(r_{\bmu^\ell}; \Gamma^\ell, D_\ell), \\
    r_{\bmu^\ell} =& r(\bmu^\ell, \Psi_1^\ell(\bK_{MM}^\ell)^{-1}\bM^\ell) + \tr( \tp{(\bM^\ell)}(\bK_{MM}^\ell)^{-1}(\Psi_2^\ell-\tp{(\Psi_1^\ell)}(\Gamma^\ell)^{-1}\Psi_1^\ell)(\bK_{MM}^\ell)^{-1}\bM^\ell) \\
    &+ D_\ell[\psi_0^\ell-\tr((\bK_{MM}^\ell)^{-1}\Psi_2^\ell)] + \sum_{d=1}^{D_\ell} \tr((\bK_{MM}^\ell)^{-1}\bSigma_d^\ell (\bK_{MM}^\ell)^{-1}\Psi_2^\ell) \\
    &+ \tr((\bI_{D_\ell}\otimes (\Gamma^\ell)^{-1}) \diag(\{\bS_n^\ell\})), \\
    - \KL^*_{\bU^\ell} =& \half \sum_{d=1}^{D_\ell}\log |\bSigma_d^\ell| 
    + \varphi\left(\tr(\tp{(\bM^\ell)}(\bK_{MM}^\ell)^{-1}\bM^\ell) + \sum_{d=1}^{D_\ell}\tr(\bSigma_d^\ell(\bK_{MM}^\ell)^{-1}); \bK_{MM}^\ell, D_\ell\right) , \\
    \mH_q(\bX_\ell) \geq& \half \sum_{n=1}^N\log |\bS_n^\ell|, \\
    - \KL^*_\bZ =& \geq \half \sum_{n=1}^N\log |\bS_n^L| + \varphi\left(\tr(\tp{(\bmu^L)}\bmu^L) + \sum_{n=1}^N\tr(\bS_n^L); \bI_N, D_L\right) ,
\end{aligned}
\end{equation*}
where $\psi_0^\ell=\tr((\Gamma^\ell)^{-1}\langle\bK_{NN}^\ell\rangle_{q(\bX^{\ell+1})})$, $\Psi_1^\ell= \langle\bK_{NM}^\ell\rangle_{q(\bX^{\ell+1})}$, and $\Psi_2^\ell = \langle\bK_{MN}^\ell\bK_{NM}^\ell\rangle_{q(\bX^{\ell+1})}$.

\section{Numerical Experiments}\label{sec:numerics}

In this section, we compare our proposed deep Q-EP with deep GP \citep[DGP][]{Damianou_2013}, deep kernel learning with GP \citep[DKL-GP][]{Wilson_2016}, and deep sigma point process \citep[DSPP][]{Jankowiak_2020b} using simulated and benchmark datasets. In simulations, deep Q-EP model manifests unique features in properly modeling inhomogeneous data with abrupt changes or sharp contrast. For benchmark regression and classification problems, deep Q-EP demonstrates superior or comparable numerical performance. 
In most cases, 2 layer structure is sufficient for deep Q-EP to have superior or comparable performance compared with deep GP, and DSPP. A large feature extracting neural network (DNN with structure $D_L-1000-500-50-D_0$) is employed before one GP layer for DKL-GP unless stated otherwise. The Mat\'ern kernel ($\nu=1.5$) is adopted for all the models with trainable hyperparameters (magnitude and correlation strength) and $q=1$ is chosen for Q-EP and deep Q-EP models.
All the models are implemented in \texttt{GPyTorch} \citep{Gardner_2018} on GPU and the codes will be released.

\subsection{Time Series Regression}

We first consider a simulated 2-dimensional time series from \cite{Li_2023}, one with step jumps and the other with sharp turnings, whose true trajectories are as follows:
\begin{equation*}
\begin{aligned}
u_\text{J}(t) &=
1, && t\in[0,1]; \quad
0.5, && t\in (1, 1.5]; \quad
2, && t\in (1.5,2]; \quad
0, && otherwise; \\
u_\text{T}(t) &=
1.5t, && t\in[0,1]; \quad
3.5-2t, && t\in (1, 1.5]; \quad
3t-4, && t\in (1.5,2]; \quad
0, && otherwise.
\end{aligned}
\end{equation*}


We generate time series $\{\by_i\}_{i=1}^N$ by adding Gaussian noises to the true trajectories evaluated at $N=100$ evenly spaced points $t_i\in[0,2]$, i.e., 
$\by_{i}^*=\tp{[u_J(t_i), u_T(t_i)]}+ \vect\eps_i, \; \vect\eps_i\overset{iid}{\sim} N(\bzero,\sigma^2\bI_2), \;with \; \sigma=0.1, \; i=1,\cdots, N.$
Then we make prediction over 50 points evenly spread over $[0,2]$.

\begin{figure}[tbp]
\begin{subfigure}[b]{.495\textwidth}
\includegraphics[width=1\textwidth,height=.3\textwidth]{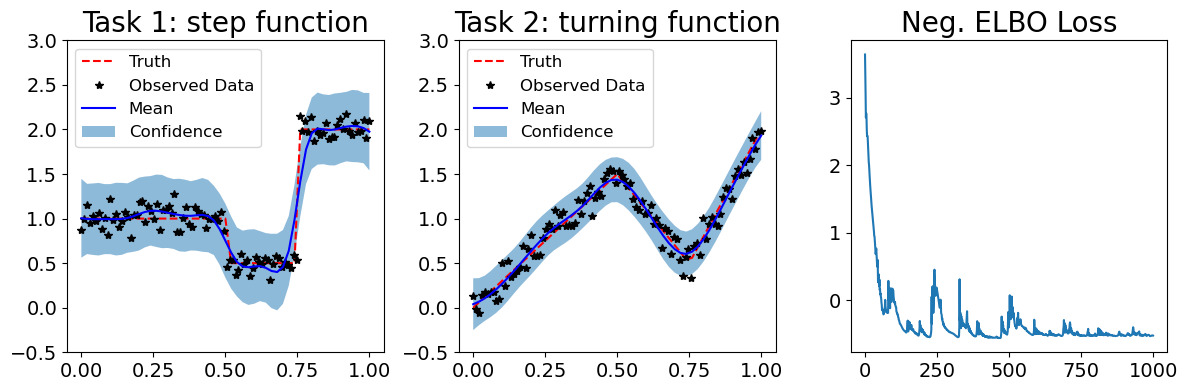}
\caption{Shallow GP regression.}
\label{fig:ts_GP}
\end{subfigure}
\begin{subfigure}[b]{.495\textwidth}
\includegraphics[width=1\textwidth,height=.3\textwidth]{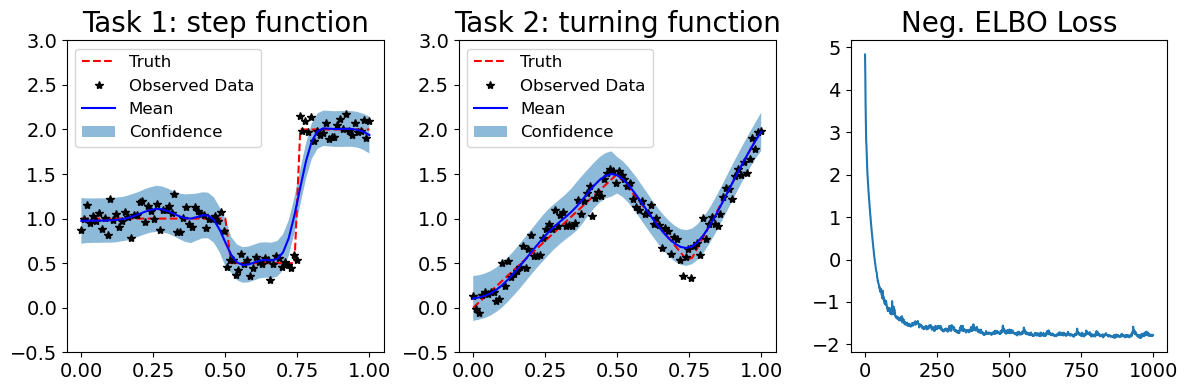}
\caption{Shallow Q-EP regression.}
\label{fig:ts_QEP}
\end{subfigure}
\begin{subfigure}[b]{.495\textwidth}
\includegraphics[width=1\textwidth,height=.3\textwidth]{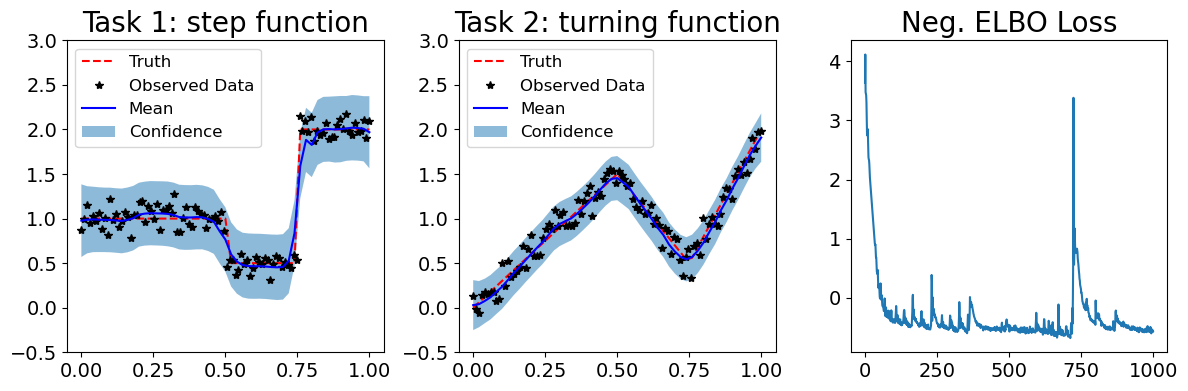}
\caption{Deep GP regression.}
\label{fig:ts_DGP}
\end{subfigure}
\begin{subfigure}[b]{.495\textwidth}
\includegraphics[width=1\textwidth,height=.3\textwidth]{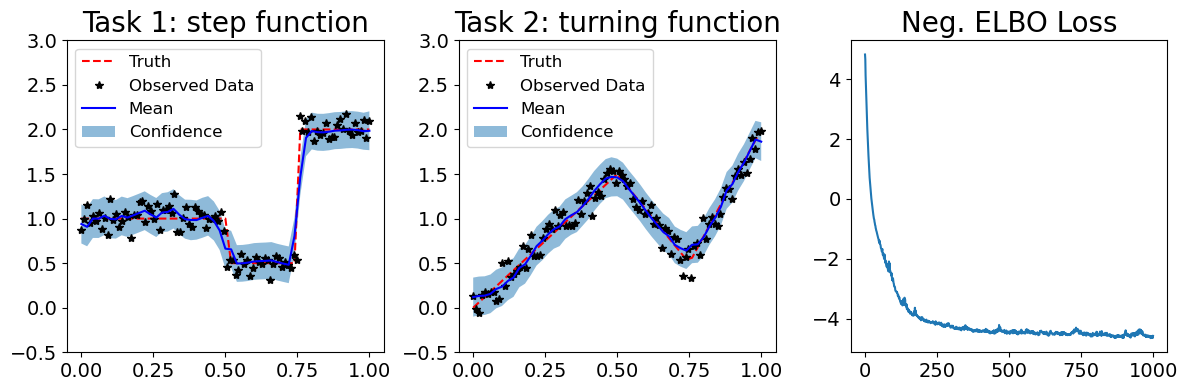}
\caption{Deep Q-EP regression.}
\label{fig:ts_DQEP}
\end{subfigure}
\begin{subfigure}[b]{.495\textwidth}
\includegraphics[width=1\textwidth,height=.3\textwidth]{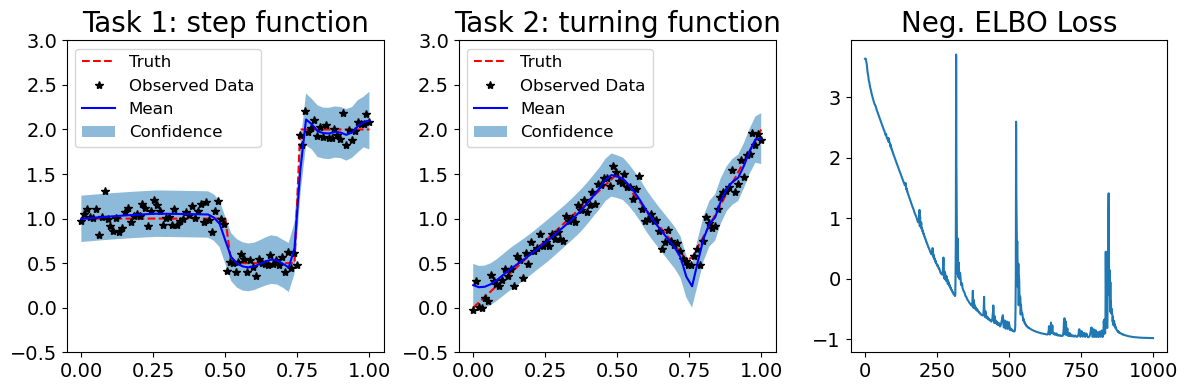}
\caption{DKL-GP regression.}
\label{fig:ts_DKLGP}
\end{subfigure}
\begin{subfigure}[b]{.495\textwidth}
\includegraphics[width=1\textwidth,height=.3\textwidth]{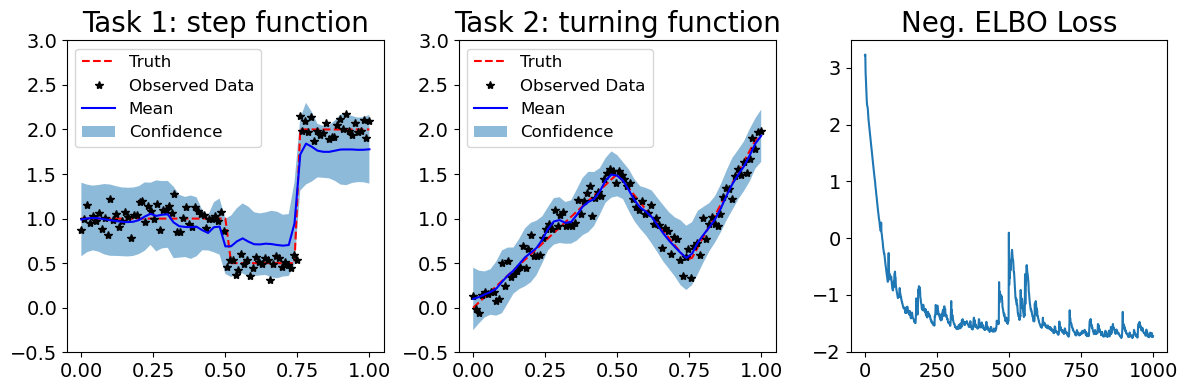}
\caption{DSPP regression.}
\label{fig:ts_DSPP}
\end{subfigure}
\caption{Comparing deep Q-EP \eqref{fig:ts_DQEP} with cutting-edge deep models including deep GP \eqref{fig:ts_DGP}, DKL-GP \eqref{fig:ts_DKLGP} and DSPP \eqref{fig:ts_DSPP} on modeling a 2d-output time series. Mean absolute errors (MAE) on testing data are 0.0494 (shallow GP), 0.0578 (shallow Q-EP), 0.0442 (deep GP), 0.0444 (deep Q-EP), 0.0536 (DKL-GP), 0.0896 (DSPP) respectively.}
\label{fig:ts_2doutput}
\end{figure}

These time series have abrupt changes in either values or directions, hence may pose challenges for standard GP as an $L_2$ penalty based regression method. 
As shown in Figure \ref{fig:ts_2doutput},
all the deep probabilistic models are comparatively better than their shallow (one-layer) versions. Among these models, deep Q-EP yields the most accurate prediction and the tightest uncertainty bound. The loss of deep Q-EP may not be comparable to that for other models because they are based on different probability distributions, 
and yet it converges fast and stably. DKL-GP converges the slowest and it suffers from unstable training issues. 

\subsection{UCI Regression Dataset}
Next, we test deep Q-EP on a series of benchmark regression datasets \citep{Wilson_2016,Jankowiak_2020b} from UCI machine learning repository. They are selected to represent small data (\emph{gas} with $N=2565$, $d=128$) and big data (\emph{song} with $N=515345$, $d=90$) cases. As in Table \ref{tab:uci_reg},  for most cases, deep Q-EP demonstrates superior or comparable performance measured by mean absolute error (MAE), standard deviation (STD) and negative log-likelihood (NLL) because the Q-EP prior provides crucial information. As the data volume increases, DNN feature extractor starts to catch up so that DKL-GP surpasses the vanilla deep Q-EP in the song dataset. Note, the GP component of DKL can be replaced with Q-EP to regularize the model. In our experiment, the resulting DKL-QEP beats DKL-GP with (MAE, STD, NLL)$=(0.327, 0.009, 0.59)$ on the \emph{protein} dataset and has comparable results on other big datasets. We will explore DKL-QEP further in Section \ref{sec:imgcls}.

\begin{table}[htbp]\tiny
\caption{\small Regression on UCI datasets: mean of absolute error (MAE), standard deviation (STD) and negative log-likelihood (NLL) values  by various deep models. Each result of the upper part is averaged over 10 experiments with different random seeds; values in the lower part are standard errors of these repeated experiments.}
\centering
\begin{tabular}{p{20pt}|p{27pt}|lll|lll|lll|lll}
\toprule
\multicolumn{2}{c|}{} & \multicolumn{3}{|c}{ Deep GP} & \multicolumn{3}{|c}{Deep Q-EP} & \multicolumn{3}{|c}{DKL-GP}  & \multicolumn{3}{|c}{DSPP}\\
\cmidrule{1-2} \cmidrule{3-5} \cmidrule{6-8} \cmidrule{9-11} \cmidrule{12-14}
Dataset &  N, d & MAE & STD & NLL &   MAE & STD & NLL &   MAE & STD & NLL &   MAE & STD & NLL \\
\midrule
gas & 2565, 128 & 0.19 & 0.06 & 0.4 & {\bf 0.14} & {\bf 0.03} & {\bf -0.6} & 0.93 & 0.07 & 2.23 & 0.33 & 0.35 & 18.54 \\
parkinsons & 5875, 20 & {\bf 8.17} &  0.61 & 168.12   & 8.49 &  {\bf 0.38} & 13  & 10.01 & 0.57 & {\bf 11.82} & 9.63 & 0.84 & 549.92 \\
elevators & 16599, 18 & 0.0639 & 0.014 & {\bf -1.04} & {\bf 0.0636} & {\bf 0.011} & -0.87 & 0.099 & 0.02 & -0.29 & 0.09 & 0.09 & 0.52 \\
protein & 45730, 9 & 0.39 & 0.05 & 0.76 & {\bf 0.35} & {\bf 0.014} & {\bf 0.7} & 0.37 & 0.02 & 0.77 & 0.48 & 0.21 & 100.66 \\
song & 515345, 90 & 0.38 & 0.011 & 0.69 & 0.4 & 0.011 & 0.92 & {\bf 0.35} & {\bf 0.008} & {\bf 0.63} & 0.43 & 0.2 & 261.3 \\
\midrule
gas & 2565, 128 & 0.07 & 0.02 & 0.16 & 0.03 & 0.01 & 0.24 & 0.36 & 0.02 & 1.04 & 0.24 & 0.13 & 22.06 \\
parkinsons & 5875, 20 & 1.38 &  0.16 & 97.06   & 1.74 &  0.11 & 3.42  & 1.55 & 0.25 & 4.89 & 1.51 & 0.29 & 349.22 \\
elevators & 16599, 18 & 3e-4 & 3e-4 & 7e-3 & 4e-4 & 3e-5 & 6e-3 & 0.06 & 0.05 & 1.32 & 0.02 & 0.02 & 0.64 \\
protein & 45730, 9 & 5e-3 & 4e-3 & 7e-3 & 5e-3 & 5e-4 & 0.01 & 0.09 & 6e-3 & 0.19 & 0.04 & 0.02 & 52.21 \\
song & 515345, 90 & 2e-3 & 1e-9 & 4e-3 & 0.04 & 3e-4 & 0.09 & 4e-3 & 1e-3 & 0.01 & 0.03 & 0.05 & 266.2 \\
\bottomrule
\end{tabular}
\label{tab:uci_reg}
\end{table}

\subsection{Classification}

Now we consider a simulated classification problem with labels created on annular regions of a rhombus:
\begin{equation*}
y_i = [\cos(0.4 * u * \pi \Vert\bx_i\Vert_1 ) ]+1, \quad u \sim \mathrm{Unif}[0,1], \quad \bx_i \sim \mN(\bzero, \bI_2), \quad i=1,\cdots, N,
\end{equation*}
where $[x]$ rounds $x$ to the nearest integer. We generate $N=500$ random data points according to the formula which results in 3 classes' labels as illustrated in the leftmost panel of Figure \ref{fig:cls_bdy}. Note, the class regions have clear shapes with edges and are not simply connected. Q-EP and deep Q-EP are superior than their GP rivals in modeling such inhomogeneous data. 
Indeed, Figure \ref{fig:cls_bdy} shows that even with small amount of data, Q-EP has better decision boundaries than GP and a 3-layer deeper Q-EP yields the best result closest to the truth among all the deep probabilistic models . This is further illustrated in Figure \ref{fig:cls_logits} where more fine details are revealed by the logits.


\begin{figure}[tp]
\includegraphics[width=.242\textwidth,height=.22\textwidth]{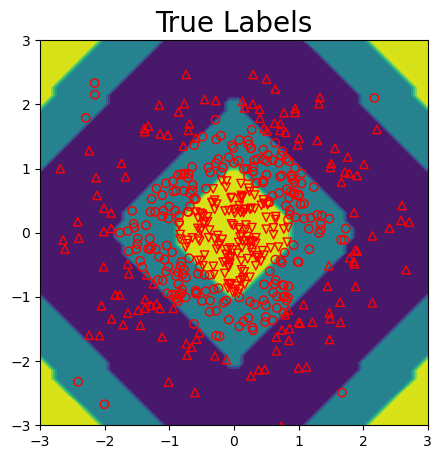}
\includegraphics[width=.242\textwidth,height=.22\textwidth]{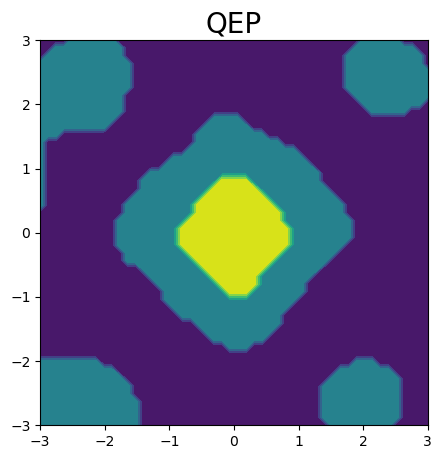}
\includegraphics[width=.242\textwidth,height=.22\textwidth]{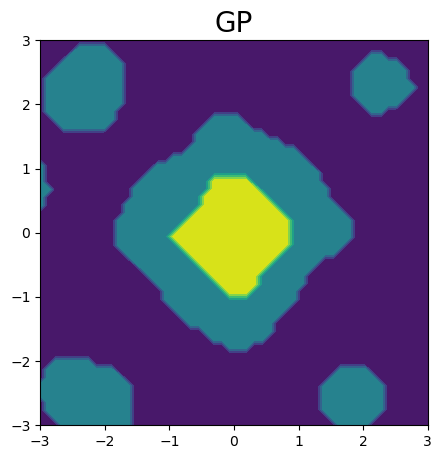}
\includegraphics[width=.242\textwidth,height=.22\textwidth]{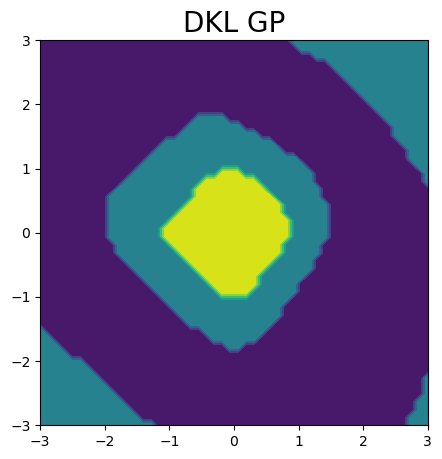}\\
\includegraphics[width=.242\textwidth,height=.22\textwidth]{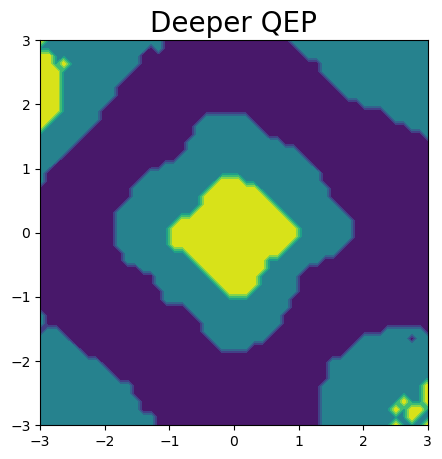}
\includegraphics[width=.242\textwidth,height=.22\textwidth]{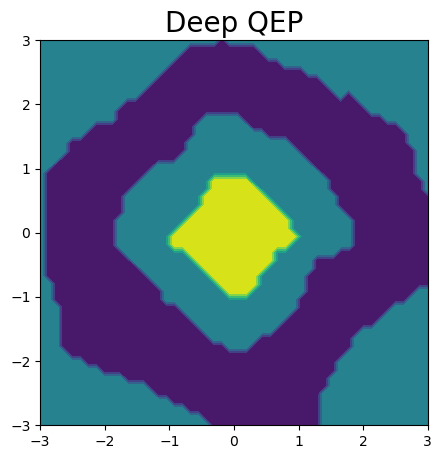}
\includegraphics[width=.242\textwidth,height=.22\textwidth]{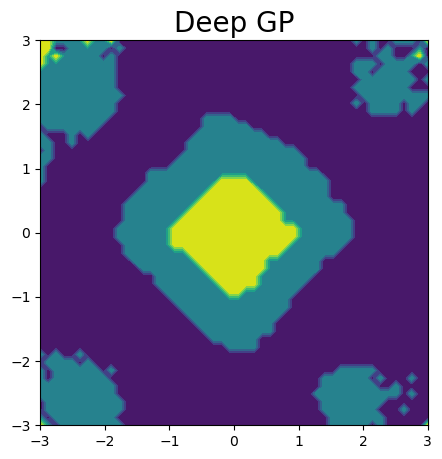}
\includegraphics[width=.242\textwidth,height=.22\textwidth]{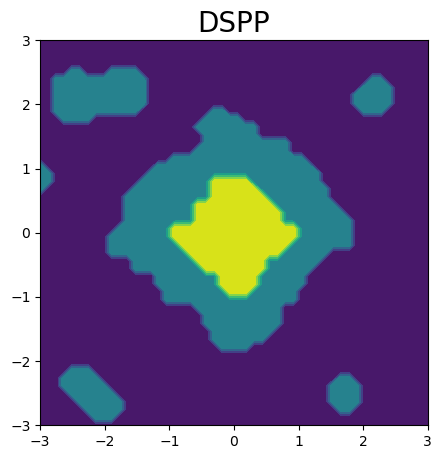}
\caption{Comparing shallow (1-layer), deep (2-layer) and deeper (3-layer) Q-EPs with GP, deep GP, DKL-GP and DSPP on a classification problem defined on annular rhombus. Circles, upper and lower triangles label three classes in the training data. Classification accuracy on testing data are $81.04\%$ (GP),  $82.2\%$ (Deep GP), $76.4\%$ (DKL-GP), $78.88\%$ (DSPP), $83.4\%$ (Q-EP), $85.64\%$ (Deep Q-EP) and $87.2\%$ (Deeper Q-EP) respectively.} 
\label{fig:cls_bdy}
\end{figure}

\begin{table}[htbp]\tiny
\caption{\small Classification on UCI datasets: accuracy (ACC), area under ROC curve (AUC) and negative log-likelihood (NLL) values  by various deep models. Each result of the upper part is averaged over 10 experiments with different random seeds; values in the lower part are standard errors of these repeated experiments.}
\centering
\begin{tabular}{p{25pt}|p{30pt}|lll|lll|lll|lll}
\toprule
\multicolumn{2}{c|}{} & \multicolumn{3}{|c}{ Deep GP} & \multicolumn{3}{|c}{Deep Q-EP} & \multicolumn{3}{|c}{DKL-GP}  & \multicolumn{3}{|c}{DSPP}\\
\cmidrule{1-2} \cmidrule{3-5} \cmidrule{6-8} \cmidrule{9-11} \cmidrule{12-14}
Dataset &  N, d, k & ACC & AUC & NLL &   ACC & AUC & NLL &   ACC & AUC & NLL &   ACC & AUC & NLL \\
\midrule
haberman & 306, 3, 2 & 0.727 &  0.46 & 7.16   & {\bf 0.732} &  {\bf 0.505} & {\bf 6.44}  & 0.702 & 0.43 & 6.93 & 0.716 & 0.496 & 31.58 \\
tic-tac-toe & 957, 27, 2 & 0.971 & 0.52 & 67.57 & {\bf 0.972} & 0.53 & 48.69 & 0.922 & {\bf 0.67} & 15.8 & 0.736 & 0.5 & 430.25 \\
car & 1728, 21, 4 & {\bf 0.99} & {\bf 0.9999} & 501.9 & 0.983 & 0.999 & 1237.08 & 0.929 & 0.98 & {\bf 46.71} & 0.758 & 0.85 & 4.6e4 \\
seismic & 2583, 24, 2 & 0.931 & 0.28 & 11.75 & {\bf 0.934} & 0.44 & 10.69 & 0.931 & 0.44 & {\bf 9.43} & 0.849 & {\bf 0.52} & 3.7e4 \\
nursery & 12959, 27, 5 & {\bf 0.9996} & {\bf 0.97} & 2.1e5 & {\bf 0.9996} & 0.95 & 1.1e4 & 0.486 & 0.7 & {\bf 2.7e3} & 0.717 & 0.84 & 1.5e5 \\
\midrule
haberman & 306, 3, 2 & 0.01 &  0.08 & 0.68   & 0.02 &  0.07 & 0.61  & 0.04 & 0.09 & 1 & 0.03 & 0.05 & 50.73 \\
tic-tac-toe & 957, 27, 2 & 0.02 & 0.08 & 20.68 & 0.04 & 0.37 & 13.25 & 0.19 & 0.15 & 4.22 & 0.23 & 0.44 & 73.5 \\
car & 1728, 21, 4 & 9e-3 & 2e-4 & 65.67 & 7e-3 & 1e-3 & 572.46 & 0.09 & 0.03 & 15.41 & 0.22 & 0.18 & 2.6e4 \\
seismic & 2583, 24, 2 & 0.002 & 0.02 & 1.25 & 0.0 & 0.1 & 0.9 & 0.006 & 0.08 & 1.48 & 0.27 & 0.13 & 1.7e4 \\
nursery & 12959, 27, 5 & 6e-8 & 0.04 & 4.5e4 & 6e-8 & 0.03 & 2.5e3 & 0.36 & 0.31 & 6e3 & 0.18 & 0.08 & 1e5 \\
\bottomrule
\end{tabular}
\label{tab:uci_cls}
\end{table}

We also compare deep Q-EP with other deep probabilistic models on several benchmark classification datasets with different sizes from UCI machine learning repository.
Table \ref{tab:uci_cls} summarizes the comparison results in terms of accuracy (ACC), area under the curve (AUC) of receiver operating characteristic (ROC) and NLL. Deep Q-EP still excels in most cases or has comparable performance, further supporting its advantage in the classification task.

\subsection{Image Classification}\label{sec:imgcls}

Finally, we test the proposed models on some benchmark image classification datasets, MNIST (60,000 training and 10,000 testing $28\times 28$ handwritten digits) and CIFAR-10 (50,000 training and 10,000 testing $32\times 32$ color images with 10 classes).
As shown in Figure \ref{fig:imgcls},
while deep GP and deep Q-EP have mediocre classification accuracy, deep kernel learning \citep[DKL][]{Wilson_2016} with CNN (common structure for these benchmarks) prefixed as a feature extractor works much better in both tasks.
On MNIST dataset, DKL-GP has a $98.14\%$ and DKL-QEP achieves a $98.19\%$ test accuracy, improving vanilla CNN with $97.69\%$ accuracy.
On CIFAR-10, DKL-GP has accuracy $70\%$ and DKL-QEP improves it to $73.4\%$, both having a good margin of advantage compared with vanilla CNN with $63.46\%$.
Note, here we choose a relatively small CNN to demonstrate the improvement by adopting DKL with Q-EP even better than DKL-GP.

\begin{figure}[tbp]
\begin{subfigure}[b]{.495\textwidth}
\includegraphics[width=1\textwidth,height=.4\textwidth]{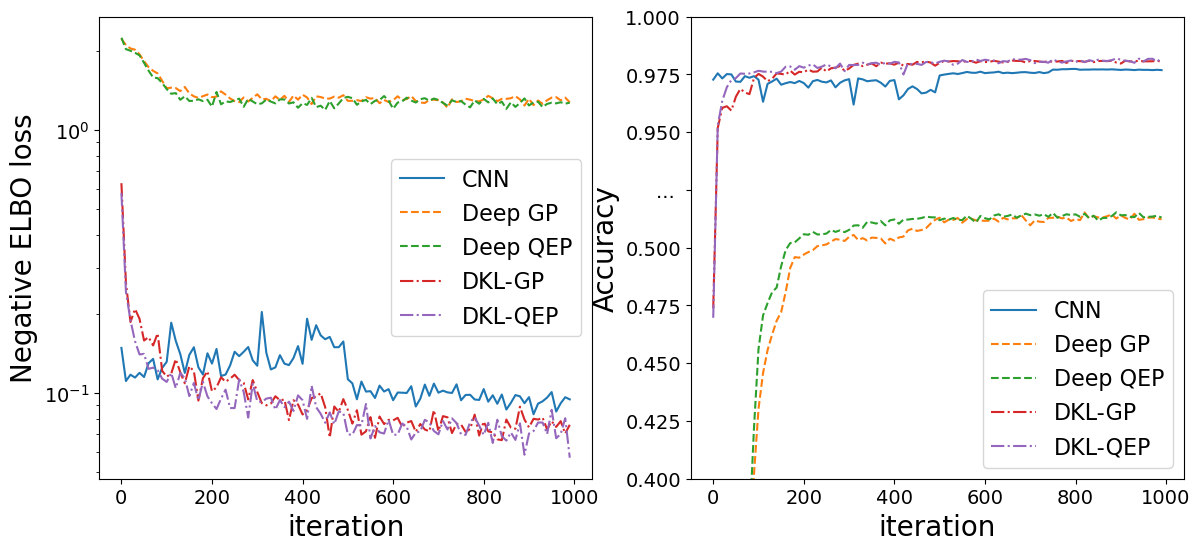}
\caption{MNIST digit classification.}
\label{fig:mnist}
\end{subfigure}
\begin{subfigure}[b]{.495\textwidth}
\includegraphics[width=1\textwidth,height=.4\textwidth]{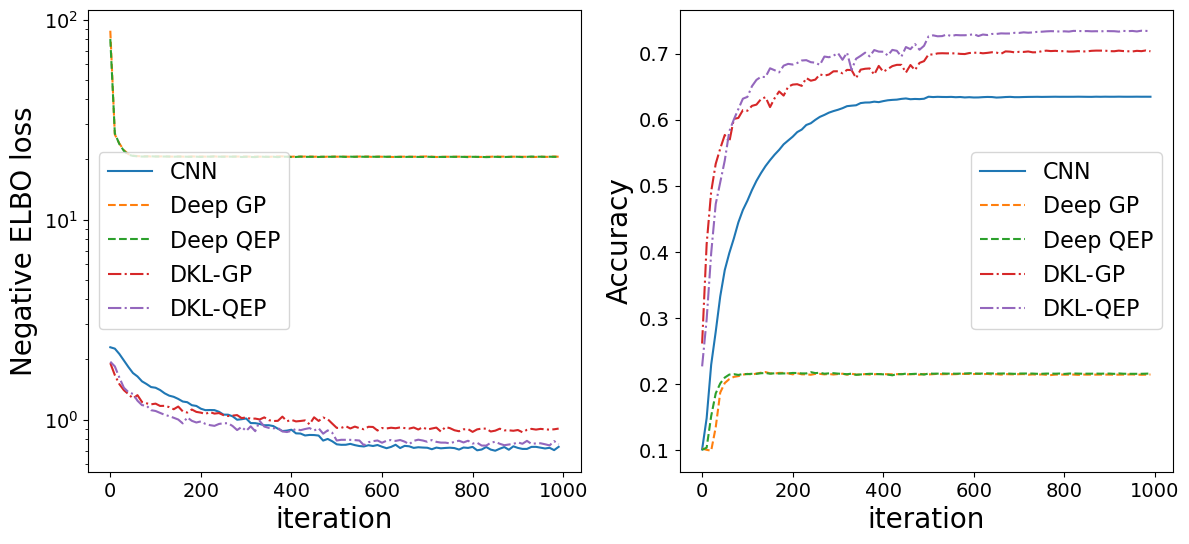}
\caption{CIFAR-10 classification.}
\label{fig:cifar10}
\end{subfigure}
\caption{Comparing DKL-QEP and DKL-GP with CNN on two benchmark classification problems.}
\label{fig:imgcls}
\end{figure}

\section{Conclusion}\label{sec:conclusion}
In this paper, we generalize Q-EP to deep Q-EP, which includes deep GP as a special case. Moreover, deep Q-EP inherits the flexible regularization controlled a parameter $q>0$, which is advantageous in learning latent representations and modeling data inhomogeneity. We first generalize Bayesian GP-LVM to Bayesian QEP-LVM (as shallow Q-EP layer) and develop the variational inference for it. Then we stack multiple shallow Q-EP layer to build the deep Q-EP model. The novel deep model demonstrates substantial numerical benefits in various learning tasks and can be combined with neural network for better characterizing the latent representation of big data applications. 

As common in GP and NN models, we do observe multi-modality of the posterior distributions, especially in the hyper-parameter spaces. Sub-optimal solutions can appear in the stochastic training process. These issues can be alleviated by dispersed or diversified initialization, or with adaptive training schedulers.
One potential application of deep Q-EP is the inverse learning, similarly as done by deep GP \citep{Jin_2015,Abraham_2023}. Theory of the contraction properties \citep{Finocchio_2023} is also an interesting research direction.



\clearpage

\bibliographystyle{plain}
\bibliography{ref}

\begin{thebibliography}{10}

\bibitem{Abraham_2023}
Kweku Abraham and Neil Deo.
\newblock Deep gaussian process priors for bayesian inference in nonlinear
  inverse problems.
\newblock 12 2023.

\bibitem{Aitchison_2021}
Laurence Aitchison, Adam Yang, and Sebastian~W Ober.
\newblock Deep kernel processes.
\newblock In Marina Meila and Tong Zhang, editors, {\em Proceedings of the 38th
  International Conference on Machine Learning}, volume 139 of {\em Proceedings
  of Machine Learning Research}, pages 130--140. PMLR, 18--24 Jul 2021.

\bibitem{Blomqvist_2020}
Kenneth Blomqvist, Samuel Kaski, and Markus Heinonen.
\newblock Deep convolutional gaussian processes.
\newblock In {\em Machine Learning and Knowledge Discovery in Databases}, pages
  582{\^a}--597. Springer International Publishing, 2020.

\bibitem{Castillo_2024}
Isma{\"e}l Castillo and Thibault Randrianarisoa.
\newblock Deep horseshoe gaussian processes.
\newblock 03 2024.

\bibitem{Damianou_2013}
Andreas Damianou and Neil~D. Lawrence.
\newblock Deep {G}aussian processes.
\newblock In Carlos~M. Carvalho and Pradeep Ravikumar, editors, {\em
  Proceedings of the Sixteenth International Conference on Artificial
  Intelligence and Statistics}, volume~31 of {\em Proceedings of Machine
  Learning Research}, pages 207--215, Scottsdale, Arizona, USA, 29 Apr--01 May
  2013. PMLR.

\bibitem{Dashti_2012}
Masoumeh Dashti, Stephen Harris, and Andrew Stuart.
\newblock Besov priors for bayesian inverse problems.
\newblock {\em Inverse Problems and Imaging}, 6(2):183--200, may 2012.

\bibitem{Dutordoir_2020}
Vincent Dutordoir, Mark van~der Wilk, Artem Artemev, and James Hensman.
\newblock Bayesian image classification with deep convolutional gaussian
  processes.
\newblock In Silvia Chiappa and Roberto Calandra, editors, {\em Proceedings of
  the Twenty Third International Conference on Artificial Intelligence and
  Statistics}, volume 108 of {\em Proceedings of Machine Learning Research},
  pages 1529--1539. PMLR, 26--28 Aug 2020.

\bibitem{fang1990generalized}
K.~Fang and Y.T. Zhang.
\newblock {\em Generalized Multivariate Analysis}.
\newblock Science Press, 1990.

\bibitem{Finocchio_2023}
Gianluca Finocchio and Johannes Schmidt-Hieber.
\newblock Posterior contraction for deep gaussian process priors.
\newblock {\em Journal of Machine Learning Research}, 24(66):1--49, 2023.

\bibitem{Gardner_2018}
Jacob Gardner, Geoff Pleiss, Kilian~Q Weinberger, David Bindel, and Andrew~G
  Wilson.
\newblock Gpytorch: Blackbox matrix-matrix gaussian process inference with gpu
  acceleration.
\newblock In S.~Bengio, H.~Wallach, H.~Larochelle, K.~Grauman, N.~Cesa-Bianchi,
  and R.~Garnett, editors, {\em Advances in Neural Information Processing
  Systems}, volume~31. Curran Associates, Inc., 2018.

\bibitem{Goodfellow-et-al-2016}
Ian Goodfellow, Yoshua Bengio, and Aaron Courville.
\newblock {\em Deep Learning}.
\newblock MIT Press, 2016.
\newblock \url{http://www.deeplearningbook.org}.

\bibitem{Hensman_2015}
James Hensman, Alexander Matthews, and Zoubin Ghahramani.
\newblock {Scalable Variational Gaussian Process Classification}.
\newblock In Guy Lebanon and S.~V.~N. Vishwanathan, editors, {\em Proceedings
  of the Eighteenth International Conference on Artificial Intelligence and
  Statistics}, volume~38 of {\em Proceedings of Machine Learning Research},
  pages 351--360, San Diego, California, USA, 09--12 May 2015. PMLR.

\bibitem{Bernardo_1998}
A.~P.~Dawid J.~M.~Bernardo, J. O.~Berger and A.~F.~M. Smith.
\newblock Regression and classification using gaussian process priors.
\newblock {\em Bayesian Statistics}, 6:475--501, 1998.

\bibitem{Jankowiak_2020b}
Martin Jankowiak, Geoff Pleiss, and Jacob Gardner.
\newblock Deep sigma point processes.
\newblock In Jonas Peters and David Sontag, editors, {\em Proceedings of the
  36th Conference on Uncertainty in Artificial Intelligence (UAI)}, volume 124
  of {\em Proceedings of Machine Learning Research}, pages 789--798. PMLR,
  03--06 Aug 2020.

\bibitem{Jankowiak_2020a}
Martin Jankowiak, Geoff Pleiss, and Jacob Gardner.
\newblock Parametric {G}aussian process regressors.
\newblock In Hal~Daum{\'e} III and Aarti Singh, editors, {\em Proceedings of
  the 37th International Conference on Machine Learning}, volume 119 of {\em
  Proceedings of Machine Learning Research}, pages 4702--4712. PMLR, 13--18 Jul
  2020.

\bibitem{Jin_2015}
Ming Jin, Andreas~C. Damianou, P.~Abbeel, and Costas~J. Spanos.
\newblock Inverse reinforcement learning via deep gaussian process.
\newblock In {\em Proceedings of the 38th Conference on Uncertainty in
  Artificial Intelligence (UAI)}, volume abs/1512.08065, 2017.

\bibitem{Johnson_1987}
Mark~E. Johnson.
\newblock {\em Multivariate Statistical Simulation}, chapter 6 Elliptically
  Contoured Distributions, pages 106--124.
\newblock Probability and Statistics. John Wiley \& Sons, Ltd, 1987.

\bibitem{Jones_2023}
Andrew Jones, F.~William Townes, Didong Li, and Barbara~E. Engelhardt.
\newblock Alignment of spatial genomics data using deep gaussian processes.
\newblock {\em Nature Methods}, 20(9):1379--1387, August 2023.

\bibitem{KOZUBOWSKI_2013}
Tomasz~J. Kozubowski, Krzysztof Podg{\'o}rski, and Igor Rychlik.
\newblock Multivariate generalized laplace distribution and related random
  fields.
\newblock {\em Journal of Multivariate Analysis}, 113:59--72, 2013.
\newblock Special Issue on Multivariate Distribution Theory in Memory of Samuel
  Kotz.

\bibitem{Lassas_2009}
Matti Lassas, Eero Saksman, and Samuli Siltanen.
\newblock Discretization-invariant bayesian inversion and besov space priors.
\newblock {\em Inverse Problems and Imaging}, 3(1):87--122, 2009.

\bibitem{Lawrence_2003}
Neil Lawrence.
\newblock Gaussian process latent variable models for visualisation of high
  dimensional data.
\newblock In S.~Thrun, L.~Saul, and B.~Sch\"{o}lkopf, editors, {\em Advances in
  Neural Information Processing Systems}, volume~16. MIT Press, 2003.

\bibitem{Lawrence_2005}
Neil Lawrence.
\newblock Probabilistic non-linear principal component analysis with gaussian
  process latent variable models.
\newblock {\em Journal of Machine Learning Research}, 6(60):1783--1816, 2005.

\bibitem{Lawrence_2007}
Neil~D. Lawrence and Andrew~J. Moore.
\newblock Hierarchical gaussian process latent variable models.
\newblock In {\em Proceedings of the 24th international conference on Machine
  learning}, ICML 2007. ACM, June 2007.

\bibitem{Li_2023}
Shuyi Li, Michael O'Connor, and Shiwei Lan.
\newblock Bayesian learning via q-exponential process.
\newblock In {\em Proceedings of the 37th Conference on Neural Information
  Processing Systems}. NeurIPS, 12 2023.
\newblock arxiv:2210.07987.

\bibitem{Li_2021}
Yikuan Li, Shishir Rao, Abdelaali Hassaine, Rema Ramakrishnan, Dexter Canoy,
  Gholamreza Salimi-Khorshidi, Mohammad Mamouei, Thomas Lukasiewicz, and Kazem
  Rahimi.
\newblock Deep bayesian gaussian processes for uncertainty estimation in
  electronic health records.
\newblock {\em Scientific Reports}, 11(1), October 2021.

\bibitem{Minka_2000}
Thomas Minka.
\newblock Automatic choice of dimensionality for pca.
\newblock In T.~Leen, T.~Dietterich, and V.~Tresp, editors, {\em Advances in
  Neural Information Processing Systems}, volume~13. MIT Press, 2000.

\bibitem{Neal_1996}
Radford~M. Neal.
\newblock {\em Bayesian Learning for Neural Networks}.
\newblock Springer New York, 1996.

\bibitem{Oksendal_2003}
Bernt {\O}ksendal.
\newblock {\em Stochastic Differential Equations}.
\newblock Springer Berlin Heidelberg, 2003.

\bibitem{ortega_2023}
Luis~A. Ortega, Simon~Rodriguez Santana, and Daniel Hern{\\'a}ndez-Lobato.
\newblock Deep variational implicit processes.
\newblock In {\em The Eleventh International Conference on Learning
  Representations}, 2023.

\bibitem{Podg_rski_2011}
Krzysztof Podg{\'{o}}rski and J{\"o}rg Wegener.
\newblock Estimation for stochastic models driven by laplace motion.
\newblock {\em Communications in Statistics - Theory and Methods},
  40(18):3281--3302, sep 2011.

\bibitem{Rasmussen_2005}
Carl~Edward Rasmussen and Christopher K.~I. Williams.
\newblock {\em Gaussian Processes for Machine Learning}.
\newblock The {MIT} Press, 2005.

\bibitem{Salimbeni_2017}
Hugh Salimbeni and Marc Deisenroth.
\newblock Doubly stochastic variational inference for deep gaussian processes.
\newblock In I.~Guyon, U.~Von Luxburg, S.~Bengio, H.~Wallach, R.~Fergus,
  S.~Vishwanathan, and R.~Garnett, editors, {\em Advances in Neural Information
  Processing Systems}, volume~30. Curran Associates, Inc., 2017.

\bibitem{Tipping_1999}
Michael~E. Tipping and Christopher~M. Bishop.
\newblock {Probabilistic Principal Component Analysis}.
\newblock {\em Journal of the Royal Statistical Society Series B: Statistical
  Methodology}, 61(3):611--622, 09 1999.

\bibitem{Titsias_2009}
Michalis Titsias.
\newblock Variational learning of inducing variables in sparse gaussian
  processes.
\newblock In David van Dyk and Max Welling, editors, {\em Proceedings of the
  Twelth International Conference on Artificial Intelligence and Statistics},
  volume~5 of {\em Proceedings of Machine Learning Research}, pages 567--574,
  Hilton Clearwater Beach Resort, Clearwater Beach, Florida USA, 16--18 Apr
  2009. PMLR.

\bibitem{Titsias_2010}
Michalis Titsias and Neil~D. Lawrence.
\newblock Bayesian gaussian process latent variable model.
\newblock In Yee~Whye Teh and Mike Titterington, editors, {\em Proceedings of
  the Thirteenth International Conference on Artificial Intelligence and
  Statistics}, volume~9 of {\em Proceedings of Machine Learning Research},
  pages 844--851, Chia Laguna Resort, Sardinia, Italy, 13--15 May 2010. PMLR.

\bibitem{Ulyanov2020}
Dmitry Ulyanov, Andrea Vedaldi, and Victor Lempitsky.
\newblock Deep image prior.
\newblock {\em International Journal of Computer Vision}, 128(7):1867--1888,
  Jul 2020.

\bibitem{Wilson_2016}
Andrew~Gordon Wilson, Zhiting Hu, Ruslan Salakhutdinov, and Eric~P. Xing.
\newblock Deep kernel learning.
\newblock In Arthur Gretton and Christian~C. Robert, editors, {\em Proceedings
  of the 19th International Conference on Artificial Intelligence and
  Statistics}, volume~51 of {\em Proceedings of Machine Learning Research},
  pages 370--378, Cadiz, Spain, 09--11 May 2016. PMLR.

\end{thebibliography}

\newpage
\appendix
\onecolumn

\counterwithin{table}{section}
\counterwithin{figure}{section}

\begin{center}
    \Large Supplement Document for ``Deep Q-Exponential Processes"
\end{center}

 \section{PROOFS}\label{apx:prob_PCA}

 \begin{proof}[Proof of Theorem \ref{thm:prob_PCA}]
 The gradients of log-likelihood \eqref{eq:log-likelihood} with respect to $\bK$ can be found as
 \begin{equation}\label{eq:gradwK}
     \frac{\partial L}{\partial \bK} = -\frac{D}{2} \bK^{-1} - \left[\frac{ND}{2}\left(\frac{q}{2}-1\right) r^{-1} -\half\frac{q}{2}r^{\frac{q}{2}-1}\right] \bK^{-1}\bY\tp{\bY}\bK^{-1}.
 \end{equation}

 The MLE for $\bX$ can be found by setting $\frac{\partial L}{\partial \bX} = 2\alpha^{-1} \frac{\partial L}{\partial \bK} \bX=0$, which leads to
 \begin{equation}\label{eq:matrixeq}
     \bX = \left[-N\left(\frac{q}{2}-1\right) r^{-1} +\frac{1}{D}\frac{q}{2}r^{\frac{q}{2}-1}\right] \bY\tp{\bY}\bK^{-1} \bX.
 \end{equation}

 Now suppose we have the formal solution for \eqref{eq:matrixeq} as $\bX=\bU \bL \bV$, where $\bL$ is an $N\times Q$ matrix whose only nonzero entries $\{l_i\}$ are on the main diagonal to be determined.
 Based on the formal solution of $\bX$, we have
 \begin{equation*}
     \bK = \bU(\alpha^{-1}\bL\tp{\bL}+\beta^{-1}\bI)\tp{\bU}, \quad r(\bY) = \tr(\bK^{-1}\bY\tp{\bY}) = \sum_{i=1}^{D\wedge Q} \lambda_i(\alpha^{-1} l_i^2+\beta^{-1})^{-1}.
 \end{equation*}
 Denote by $h(r):= -N\left(\frac{q}{2}-1\right) r^{-1} +\frac{1}{D}\frac{q}{2}r^{\frac{q}{2}-1}$. We substitute the above quantities into \eqref{eq:matrixeq} and get
 \begin{equation*}
     \bU \bL \bV = \bU \vect\Lambda (\alpha^{-1}\bL\tp{\bL}+\beta^{-1}\bI)^{-1} \bL \bV h(r) ,
 \end{equation*}
 which reduces to
 \begin{equation*}
     l_i = \lambda_i(\alpha^{-1} l_i^2+\beta^{-1})^{-1} l_i h(r), \quad i=1,\cdots, D\wedge Q.
 \end{equation*}
 Let $h(r)=c$ with $c$ to be determined. 
 Assume $l_i\neq 0$. Then we can solve $l_i=\sqrt{\alpha(c\lambda_i-\beta^{-1})}$.
 This yields $r=c^{-1} (D\wedge Q)$. Hence
 \begin{equation*}
     h(r) = -N\left(\frac{q}{2}-1\right) c (D\wedge Q)^{-1} +\frac{1}{D}\frac{q}{2}(D\wedge Q)^{\frac{q}{2}-1} c^{1-\frac{q}{2}} = c.
 \end{equation*}
 And it solves 
 $c=\left[\frac{q(D\wedge Q)^{\frac{q}{2}}}{2D(D\wedge Q) + (q-2) ND}\right]^{\frac{2}{q}}$. Hence the proof is completed.

 \end{proof}

\section{Computation of Variational Lower Bounds}
\subsection{Shallow Q-EP}\label{sec:elbo_qeplvm}
The variational lower bound for the log-evidence is
\begin{equation*}
    \log p(\bY) \geq \mL(q):= \int q(\bX)\log\frac{p(\bY|\bX)p(\bX)}{q(\bX)} d\bX = \tilde{\mL}(q) - \KL(q(\bX)\Vert p(\bX)),
\end{equation*}
where the first term $\tilde{\mL}(q)=\int q(\bX)\log p(\bY|\bX) d\bX$ is intractable and hence difficult to bound.

\subsubsection{Lower bound for the marginal likelihood}

To address such intractability issue and speed up the computation, sparse variational approximation \citep{Titsias_2009, Lawrence_2007} is adopted by introducing a set of inducing points $\tilde\bX\in\mbR^{M\times Q}$ with their function values $\bU=[f_1(\tilde\bX),\cdots,f_D(\tilde\bX)]\in\mbR^{M\times D}$.
Hence the marginal likelihood $p(\bY|\bX)$ defined in \eqref{eq:Bayes_LVM} can be augmented to the following joint distribution each being a $\qED$:
\begin{equation*}
    p(\bY|\bX)\propto p(\bY|\bF) p(\bF|\bU,\bX,\tilde\bX) p(\bU|\tilde\bX),
\end{equation*}
where we have $\VEC(\bU)|\tilde\bX\sim \qED(\bzero, \bI_D\otimes\bK_{MM})$ and the conditional distribution
\begin{equation}\label{eq:cond_qed}
    \VEC(\bF)|\bU,\bX,\tilde\bX \sim \qED(\VEC(\bK_{NM}\bK_{MM}^{-1}\bU), \bI_D\otimes(\bK_{NN}-\bK_{NM}\bK_{MM}^{-1}\bK_{MN})).
\end{equation}
The inducing points $\tilde\bX$ are regarded as variational parameters and hence they are dropped from the following probability expressions.
We then approximate $p(\bF,\bU|\bX)\propto p(\bF|\bU,\bX) p(\bU)$ with $q(\bF,\bU)=p(\bF|\bU,\bX)q(\bU)$ in another variational Bayes as follows
\begin{equation}\label{eq:marginal_elbo}
\begin{aligned}
\log p(\bY|\bX) &\geq \int q(\bF,\bU)\log \frac{p(\bY|\bF) p(\bF|\bU,\bX) p(\bU)}{q(\bF,\bU)} d\bF d\bU \\
    &= \int p(\bF|\bU) q(\bU) d\bU \log p(\bY|\bF) d\bF + \int q(\bU) \log \frac{p(\bU)}{q(\bU)} d\bU. \\
\end{aligned}
\end{equation}

Different from \cite{Titsias_2009, Titsias_2010} using the variational calculus, \cite[SVGP][]{Hensman_2015} computes the marginal likelihood ELBO \eqref{eq:marginal_elbo} in two stages.
Instead of the variational free form, we follow \cite{Hensman_2015} to use the variational distribution for $\bU$ of the following format conjugate to $p(\bF|\bU)$:
\begin{equation}\label{eq:var_dist_inducing}
    q(\bU) \sim \qED(\bM, \diag(\{\bSigma_d\})).
\end{equation}
Noticing that $\bF|\bU$ follows a conditional $q$-exponential \eqref{eq:cond_qed}, we can obtain the variational distribution of $\bF$, $q(\bF)$, by marginalizing $\bU$ out as follows
\begin{equation*}
\begin{aligned}
    &q(\bF) = \int q(\bF, \bU) d\bU = \int p(\bF|\bU) q(\bU) d\bU \\
    \sim &\qED(\VEC(\bK_{NM}\bK_{MM}^{-1}\bM), \\
    &\phantom{\qED(} \bI_D\otimes(\bK_{NN}-\bK_{NM}\bK_{MM}^{-1}\bK_{MN})+ \diag(\{\bK_{NM}\bK_{MM}^{-1}\bSigma_d \bK_{MM}^{-1} \bK_{MN}\}) ). 
\end{aligned}
\end{equation*}
Therefore, the variational lower bound of the marginal likelihood \eqref{eq:marginal_elbo} becomes
\begin{equation*}
    \log p(\bY|\bX) \geq \langle \log p(\bY|\bF) \rangle_{q(\bF)} - \KL(q(\bU) \Vert p(\bU)).
\end{equation*}

Denote by $\log p(\bY|\bF) = \varphi(r(\bY,\bF))$, where $\varphi(r):=\frac{DN}{2}\log \beta + \frac{ND}{2}\left(\frac{q}{2}-1\right) \log r - \half r^{\frac{q}{2}}$ is convex for $q\in (0,2]$, and $r(\bY,\bF) = \tp{\VEC(\bY-\bF)} (\beta^{-1}\bI_{ND})^{-1}\VEC(\bY-\bF)= \beta\tr((\bY-\bF)\tp{(\bY-\bF)})$ is a quadratic form of random variable $\bY$.
Therefore, by Jensen's inequality, we can bound from below as
\begin{equation*}
    \langle \log p(\bY|\bF) \rangle_{q(\bF)} = \langle \varphi(r(\bY,\bF)) \rangle_{q(\bF)} 
    \geq \varphi(\langle r(\bY,\bF)\rangle_{q(\bF)}).
\end{equation*}
where we can calculate the expectation of the quadratic form $r(\bY, \bF)$ as 
\begin{equation*}
\begin{aligned}
    \langle r(\bY,\bF)\rangle_{q(\bF)} =& r(\bY, \bK_{NM}\bK_{MM}^{-1}\bM) + \beta D\tr(\bK_{NN}-\bK_{NM}\bK_{MM}^{-1}\bK_{MN}) \\
    &+ \beta\sum_{d=1}^D \tr(\bK_{NM}\bK_{MM}^{-1}\bSigma_d \bK_{MM}^{-1}\bK_{MN}).
\end{aligned}
\end{equation*}

Denote by $h(\bY, \bX) = \langle \langle \log p(\bY|\bF)\rangle_{q(\bF)} \rangle_{q(\bX)}$. Then by Jensen's inequality again we have
\begin{equation*}
    h(\bY, \bX) \geq \varphi(\langle \langle r(\bY,\bF)\rangle_{q(\bF)} \rangle_{q(\bX)}) =: h^*(\bY, \bX).
\end{equation*}
Define $\psi_0=\tr(\langle\bK_{NN}\rangle_{q(\bX)})$, $\Psi_1= \langle\bK_{NM}\rangle_{q(\bX)}$, and $\Psi_2 = \langle\bK_{MN}\bK_{NM}\rangle_{q(\bX)}$.
Further we calculate the expectations of quadratic terms similarly
\begin{equation}\label{eq:condquadint_twostage}
\begin{aligned}
    \langle\langle r(\bY,\bF)\rangle_{q(\bF)}\rangle_{q(\bX)} =& \langle r(\bY, \bK_{NM}\bK_{MM}^{-1}\bM)\rangle_{q(\bX)} + \beta D[\psi_0-\tr(\bK_{MM}^{-1}\Psi_2)]\\
    & + \beta\sum_{d=1}^D \tr(\bK_{MM}^{-1}\bSigma_d \bK_{MM}^{-1}\Psi_2),\\
    \langle r(\bY, \bK_{NM}\bK_{MM}^{-1}\bM)\rangle_{q(\bX)} =& r(\bY, \Psi_1\bK_{MM}^{-1}\bM) + \beta\tr( \tp{\bM}\bK_{MM}^{-1}(\Psi_2-\tp{\Psi}_1\Psi_1)\bK_{MM}^{-1}\bM).
\end{aligned}
\end{equation}


We also need to compute the K-L divergence $\KL_\bU:=\KL(q(\bU)\Vert p(\bU))$
\begin{equation*}
    \KL_\bU = \int q(\bU) \log q(\bU) d\bU - \int q(\bU) \log p(\bU) d\bU = -\mH_q(\bU) - \langle \log p(\bU)\rangle_{q(\bU)}.
\end{equation*}
Denote by $r=\tp{\tp{\VEC}(\bU-\bM)}\diag(\{\bSigma_d\})^{-1}\tp{\VEC}(\bU-\bM)$. Then $\log q(\bU)=-\half \sum_{d=1}^D\log |\bSigma_d| + \frac{MD}{2}\left(\frac{q}{2}-1\right) \log r - \half r^{\frac{q}{2}}$. From 
\citep[Proposition A.1. of][]{Li_2023}
we know that $r^{\frac{q}{2}}\sim \chi^2(MD)$. Therefore
\begin{equation*}
\begin{aligned}
    &\mH_q(\bU) = \half \sum_{d=1}^D\log |\bSigma_d| + \frac{MD}{2}\left(\frac{q}{2}-1\right)\frac{2}{q} \mH(\chi^2(MD)) + \frac{MD}{2}\\
    &= \half \sum_{d=1}^D\log |\bSigma_d| + \frac{MD}{2}\left(1-\frac{2}{q}\right)\left[\frac{MD}{2}+\log\left( 2\Gamma\left(\frac{MD}{2}\right)\right)+\left(1-\frac{MD}{2}\right)\psi\left(\frac{MD}{2}\right)\right] + \frac{MD}{2}.
\end{aligned}
\end{equation*}
Denote by $\varphi_0(r):=-\frac{D}{2}\log |\bK_{MM}| + \frac{MD}{2}\left(\frac{q}{2}-1\right) \log r - \half r^{\frac{q}{2}}$. Then by Jensen's inequality
\begin{equation*}
\begin{aligned}
    \langle \log p(\bU)\rangle_{q(\bU)} &= \langle \varphi_0(\tr(\tp{\bU}\bK_{MM}^{-1}\bU))\rangle_{q(\bU)} \geq \varphi_0(\langle\tr(\tp{\bU}\bK_{MM}^{-1}\bU)\rangle_{q(\bU)}), \\
    \langle\tr(\tp{\bU}\bK_{MM}^{-1}\bU)\rangle_{q(\bU)} &= \tr(\tp{\bM}\bK_{MM}^{-1}\bM) + \sum_{d=1}^D\tr(\bSigma_d\bK_{MM}^{-1}).
\end{aligned}
\end{equation*}

The elements of $\psi_0$, $\Psi_1$ and $\Psi_2$ can be computed as
\begin{equation*}
\begin{aligned}
    \psi_0^n &= \int k(\bx_n, \bx_n) \qED(\bx_n |\bmu_n, \bS_n) d\bx_n, \\
    (\Psi_1)_{nm} &= \int k(\bx_n, \bz_m) \qED(\bx_n |\bmu_n, \bS_n) d\bx_n, \\
    (\Psi_2^n)_{mm'} &= \int k(\bx_n, \bz_m)k(\bz_{m'}, \bx_n) \qED(\bx_n |\bmu_n, \bS_n) d\bx_n.
\end{aligned}
\end{equation*}
With ARD SE kernel \eqref{eq:ARD}, we have $\psi_0=N\alpha^{-1}$.
While the integration in $\Psi_1$ and $\Psi_2$ is intractable in general,
we can compute them using Monte Carlo approximation.
Alternatively, we approximate
\begin{equation*}
\begin{aligned}
    (\Psi_1)_{nm} &\approx \alpha^{-1} \exp\left\{-\half \langle \tp{(\bx_n-\bz_m)} \diag(\bgamma) (\bx_n-\bz_m)\rangle_{q(\bx_n)} \right\} \\
    &= \alpha^{-1} \exp\left\{-\half [\tp{(\bmu_n-\bz_m)} \diag(\bgamma) (\bmu_n-\bz_m) + \tr(\diag(\bgamma)\bS_n)]\right\}, \\
    (\Psi_2^n)_{mm'} &\approx \alpha^{-2} \exp\left\{-\half \sum_{\tilde{m}=m, m'}\tp{(\bmu_n-\bz_{\tilde m})} \diag(\bgamma) (\bmu_n-\bz_{\tilde m})) + \tr(\diag(\bgamma)\bS_n)\right\}. \\
\end{aligned}
\end{equation*}
If we use the ARD linear form, $k(\bx, \bx') = \tp{\bx} \diag(\bgamma) \bx'$, then we have
\begin{equation*}
\begin{aligned}
     \psi_0^n &= \tr(\diag(\bgamma)(\bmu_n\tp{\bmu}_n+\bS_n)), \quad (\Psi_1)_{nm} = \tp{\bmu}_n\diag(\bgamma)\bz_m, \\
     (\Psi_2^n)_{mm'} &= \tp{\bz}_m\diag(\bgamma)(\bmu_n\tp{\bmu}_n+\bS_n)\diag(\bgamma) \bz_{m'}.
\end{aligned}
\end{equation*}

\subsubsection{Lower bound for the K-L divergence added terms}

Lastly, we need to compute the K-L divergence
\begin{equation*}
    \KL(q(\bX)\Vert p(\bX)) = \int q(\bX) \log q(\bX) d\bX - \int q(\bX) \log p(\bX) d\bX = -\mH_q(\bX) - \langle \log p(\bX)\rangle_{q(\bX)}.
\end{equation*}
Denote by $r=\tp{\VEC(\bX-\bmu)}\diag(\{\bS_n\})^{-1}\VEC(\bX-\bmu)$. Then $\log q(\bX)=-\half \sum_{n=1}^N\log |\bS_n| + \frac{NQ}{2}\left(\frac{q}{2}-1\right) \log r - \half r^{\frac{q}{2}}$. From 
\citep[Proposition A.1. of][]{Li_2023}
we know that $r^{\frac{q}{2}}\sim \chi^2(NQ)$. Therefore
\begin{equation*}
\begin{aligned}
    &\mH_q(\bX) = \half \sum_{n=1}^N\log |\bS_n| + \frac{NQ}{2}\left(\frac{q}{2}-1\right)\frac{2}{q} \mH(\chi^2(NQ)) + \frac{NQ}{2}\\
    &= \half \sum_{n=1}^N\log |\bS_n| + \frac{NQ}{2}\left(1-\frac{2}{q}\right)\left[\frac{NQ}{2}+\log\left( 2\Gamma\left(\frac{NQ}{2}\right)\right)+\left(1-\frac{NQ}{2}\right)\psi\left(\frac{NQ}{2}\right)\right] + \frac{NQ}{2}.
\end{aligned}
\end{equation*}
Denote by $\varphi_0(r):=\frac{NQ}{2}\left(\frac{q}{2}-1\right) \log r - \half r^{\frac{q}{2}}$. Then by Jensen's inequality
\begin{equation*}
\begin{aligned}
\langle \log p(\bX)\rangle_{q(\bX)} &= \langle \varphi_0(\tr(\tp{\bX}\bX))\rangle_{q(\bX)} \geq \varphi_0(\langle\tr(\tp{\bX}\bX)\rangle_{q(\bX)}), \\
\langle\tr(\tp{\bX}\bX)\rangle_{q(\bX)} &= \tr(\tp{\bmu}\bmu) + \sum_{n=1}^N\tr(\bS_n).
\end{aligned}
\end{equation*}


\subsection{Deep Q-EP}\label{sec:elbo_deepqep}
We only consider the hierarchy of two QEP-LVMs because the general $L$-layers follows by induction:
\begin{equation}\label{eq:2stacks}
\begin{aligned}
    y_{nd} &= f_d^Y(\bx_n) + \eps_{nd}^Y, \quad d=1,\cdots, D, \quad \bx_n\in \mbR^Q, \\
    x_{nq} &= f_q^X(\bz_n) + \eps_{nq}^X, \quad q=1,\cdots, Q, \quad \bz_n\in \mbR^{Q_Z},
\end{aligned}
\end{equation}
where $\vect{\eps}^Y\sim \qED(\bzero, \Gamma^Y)$, $\vect{\eps}^X\sim \qED(\bzero, \Gamma^X)$, $f^Y\sim \qEP(0, k^Y)$ and $f^X\sim \qEP(0, k^X)$.
Consider the prior $\bZ\sim \qED(\bzero, \bI_{NQ_Z})$.
The variational inference for $p(\bZ|\bY)$ requires maximizing the following ELBO
\begin{equation}\label{eq:elbo}
    \log p(\bY) \geq \mL(\mQ) := \int_{\bZ,\bF^X,\bX,\bF^Y} \mQ\log \frac{p(\bY,\bF^Y,\bX,\bF^X,\bZ)}{\mQ},
\end{equation}
where the joint probability can be decomposed
\begin{equation*}
    p(\bY,\bF^Y,\bX,\bF^X,\bZ) = p(\bY|\bF^Y) p(\bF^Y|\bX) \cdot p(\bX|\bF^X)p(\bF^X|\bZ)p(\bZ)
\end{equation*}
Similarly as in Section \ref{sec:Bayes_LVM}, sparse variational approximation \citep{Titsias_2010} is adopted to introduce inducing points $\tilde{\bX}\in\mbR^{M\times Q}, \tilde{\bZ}\in\mbR^{M\times Q_Z}$ with associated function values $\bU^Y\in\mbR^{M\times D}, \bU^X\in\mbR^{M\times Q}$ respectively. Hence the augmented probability replaces the joint probability:
\begin{equation*}
\begin{aligned}
    p(\bY,\bF^Y,\bX,\bF^X,\bZ,\bU^Y,\bU^X) =& p(\bY|\bF^Y) p(\bF^Y|\bU^Y,\bX)p(\bU^Y|\tilde{\bX}) \cdot \\
    &p(\bX|\bF^X)p(\bF^X|\bU^X,\bZ)p(\bU^X|\tilde{\bZ})p(\bZ),
\end{aligned}
\end{equation*}
where $\bF^Y$ and $\bU^Y$ are drawn from the same Q-EP; and similarly are $\bF^X$ and $\bU^X$.
Now we specify the approximation distribution as
\begin{equation*}
    \mQ=p(\bF^Y|\bU^Y, \bX) q(\bU^Y) q(\bX)\cdot p(\bF^X|\bU^X, \bZ) q(\bU^X) q(\bZ).
\end{equation*}
and choose $q(\bU^Y)$ and $q(\bU^X)$, 
and $q(\bX)$ and $q(\bZ)$ to be uncorrelated $\qED$'s: 
\begin{equation*}
\begin{aligned}
    q(\bU^Y) \sim \qED(\bM^Y, \diag(\{\bSigma_d^Y\})), &\quad q(\bU^X) \sim \qED(\bM^X, \diag(\{\bSigma_d^X\})), \\
    q(\bX) \sim \qED(\bmu^X, \diag(\{\bS^X_n\})), &\quad q(\bZ) \sim \qED(\bmu^Z, \diag(\{\bS^Z_n\})).
\end{aligned}
\end{equation*}

Then the ELBO \eqref{eq:elbo} becomes
\begin{equation*}
\begin{aligned}
    \mL(\mQ) &:= \int_{\bZ,\bU^X,\bF^X,\bX,\bU^Y,\bF^Y} \mQ\log \frac{p(\bY|\bF^Y)p(\bU^Y) p(\bX|\bF^X)p(\bU^X)p(\bZ)}{q(\bU^Y)q(\bX) q(\bU^X)q(\bZ)} \\
    &= h(\bY, \bX) - \KL_{\bU^Y} + h(\bX, \bY) - \KL_{\bU^X} +\mH_q(\bX) - \KL_\bZ,
\end{aligned}
\end{equation*}
where we have
\begin{equation*}
    h(\bY,\bX) 
    = \left\langle \log p(\bY|\bF^Y) \right\rangle_{q(\bF^Y)q(\bX)}, \quad 
    h(\bX, \bZ)
    = \left\langle \log p(\bX|\bF^X) \right\rangle_{q(\bF^X)q(\bX)q(\bZ)} .
\end{equation*}
Note, $h(\bY,\bX)\geq  h^*(\bY, \bX)$ is the same as in the bound \eqref{eq:LVM_ELBO} for Bayesian LVM. However, $h(\bX, \bZ)$ has an extra integration with respect to $q(\bX)$. Replacing $\bX$ with $\bZ$ and $\bY$ with $\bX$ in \eqref{eq:condquadint_twostage}, we compute
\begin{equation*}
    \langle r(\bX, \Psi_1(\bK^X_{MM})^{-1}\bU^X)\rangle_{q(\bX)} 
    = r(\bmu^X, \Psi_1(\bK^X_{MM})^{-1}\bU^X) 
    + \tr((\bI_D\otimes (\Gamma^X)^{-1}) \diag(\{\bS_n^X\})).
\end{equation*}
Therefore we have a updated bound for $h(\bX, \bZ)\geq h^*(\bX, \bZ)=\varphi(r_{\bmu^X}; \Gamma^X, Q)$, where
\begin{equation*}
\begin{aligned}
    r_{\bmu^X} =& r(\bmu^X, \Psi_1(\bK_{MM}^X)^{-1}\bM^X) + \tr( \tp{(\bM^X)}(\bK_{MM}^X)^{-1}(\Psi_2^X-\tp{\Psi}_1(\Gamma^X)^{-1}\Psi_1)(\bK_{MM}^X)^{-1}\bM^X) \\
    &+ Q[\psi_0-\tr((\bK_{MM}^X)^{-1}\Psi_2^X)] + \sum_{d=1}^Q \tr((\bK_{MM}^X)^{-1}\bSigma_d^X (\bK_{MM}^X)^{-1}\Psi_2^X) \\
    &+ \tr((\bI_Q\otimes (\Gamma^X)^{-1}) \diag(\{\bS_n^X\})).
\end{aligned}
\end{equation*}

Lastly, we have
\begin{equation*}
    \mH_q(\bX) \geq \half \sum_{n=1}^N\log |\bS_n^X|, \; - \KL(q(\bZ)\Vert p(\bZ)) \geq \half \sum_{n=1}^N\log |\bS_n^Z| + \varphi_0(\tr(\tp{(\bmu^Z)}\bmu^Z) + \sum_{n=1}^N\tr(\bS_n^Z)),
\end{equation*}
where $\varphi_0(r):=\frac{NQ_Z}{2}\left(\frac{q}{2}-1\right) \log r - \half r^{\frac{q}{2}}$.


\section{More Numerical Results}

\begin{figure}[htbp]
\begin{subfigure}[b]{.495\textwidth}
\includegraphics[width=1\textwidth,height=.3\textwidth]{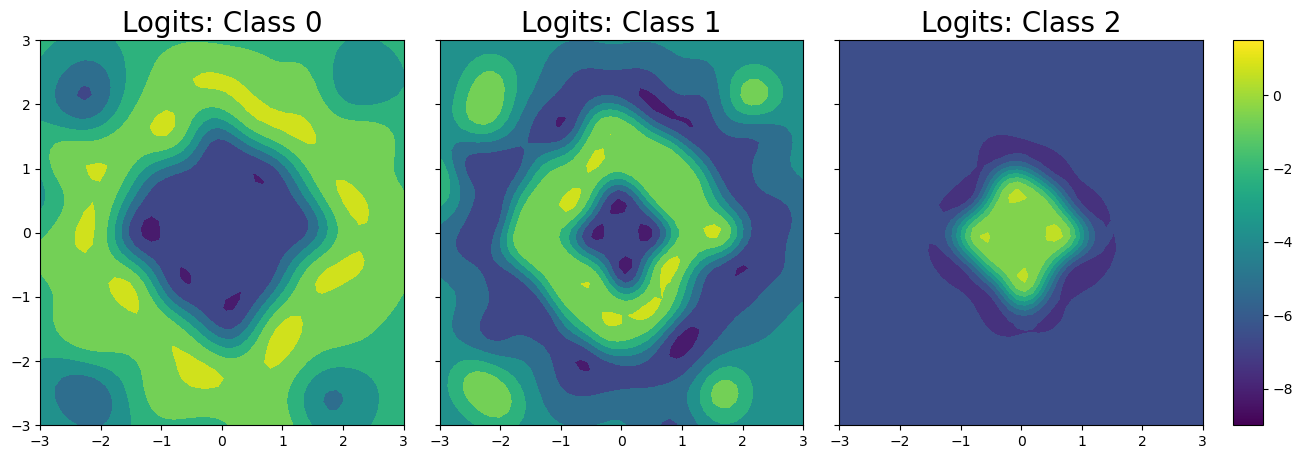}
\caption{GP classification.}
\label{fig:cls_GP_logits}
\end{subfigure}
\begin{subfigure}[b]{.495\textwidth}
\includegraphics[width=1\textwidth,height=.3\textwidth]{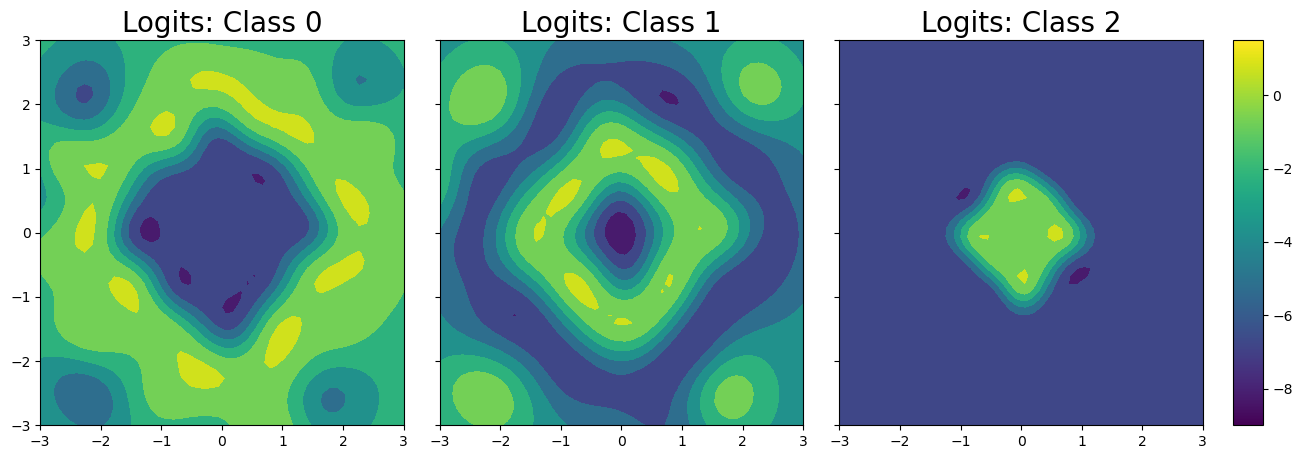}
\caption{Q-EP classification.}
\label{fig:cls_QEP_logits}
\end{subfigure}
\begin{subfigure}[b]{.495\textwidth}
\includegraphics[width=1\textwidth,height=.3\textwidth]{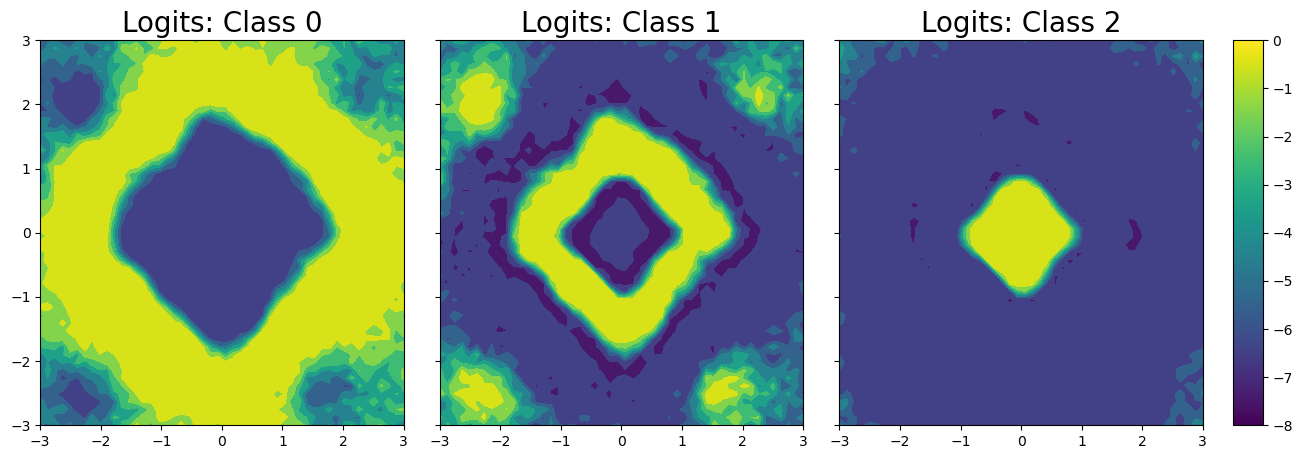}
\caption{Deep GP classification.}
\label{fig:cls_DGP_logits}
\end{subfigure}
\begin{subfigure}[b]{.495\textwidth}
\includegraphics[width=1\textwidth,height=.3\textwidth]{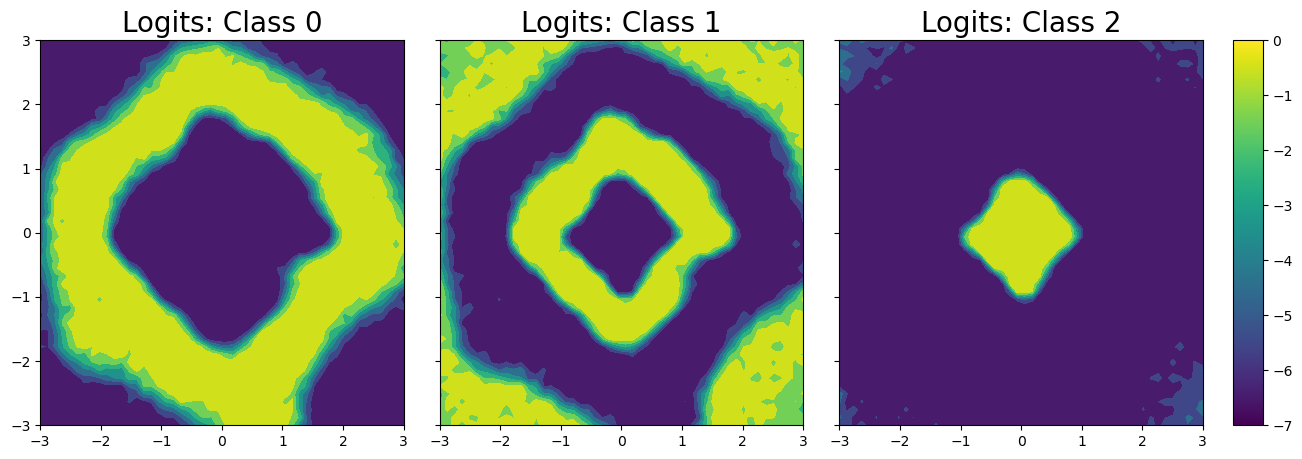}
\caption{Deep Q-EP classification.}
\label{fig:cls_DQEP_logits}
\end{subfigure}
\begin{subfigure}[b]{.495\textwidth}
\includegraphics[width=1\textwidth,height=.3\textwidth]{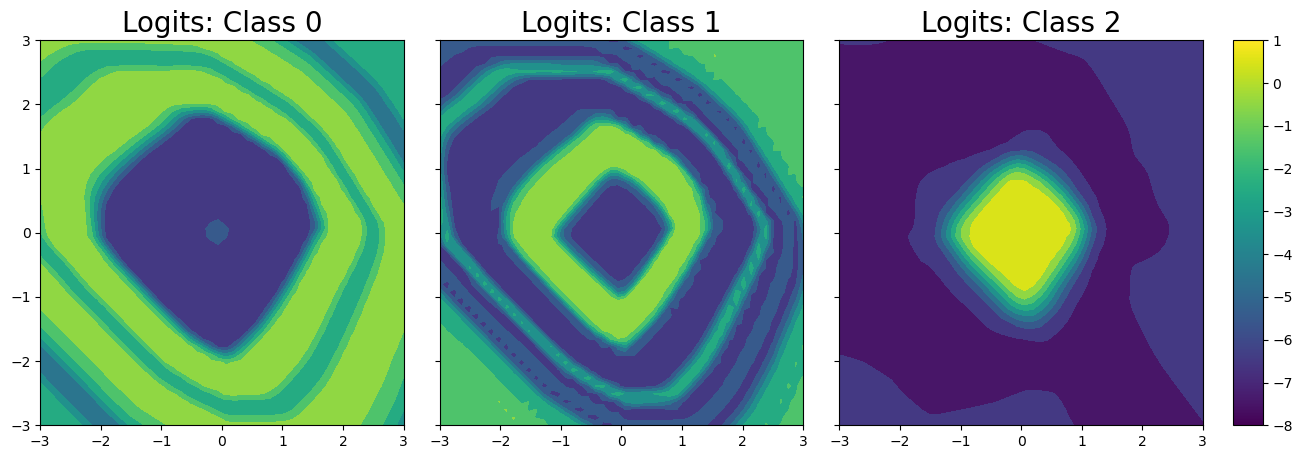}
\caption{DKL-GP classification.}
\label{fig:cls_DKLGP_logits}
\end{subfigure}
\begin{subfigure}[b]{.495\textwidth}
\includegraphics[width=1\textwidth,height=.3\textwidth]{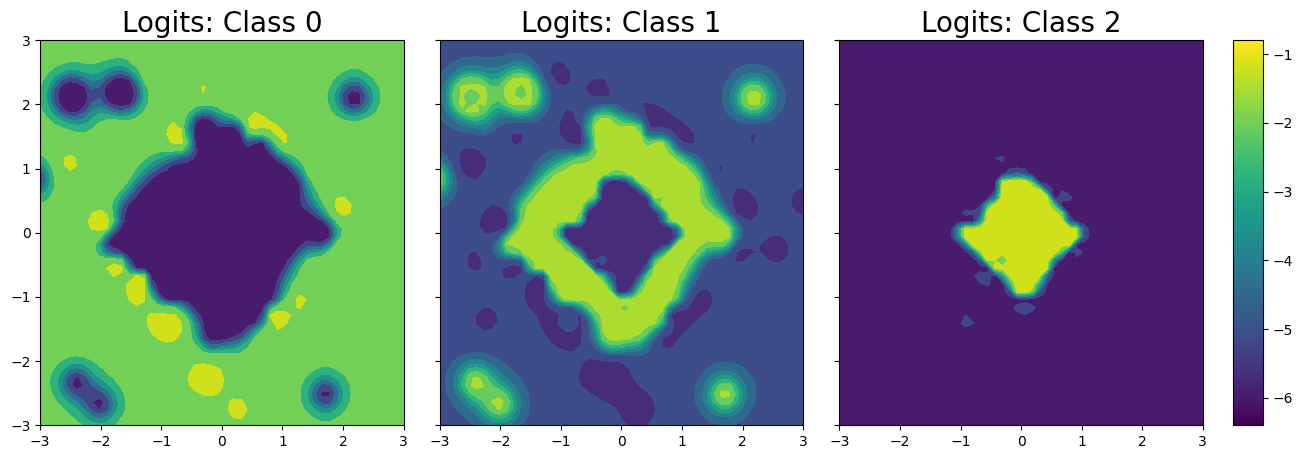}
\caption{DSPP classification.}
\label{fig:cls_DSPP_logits}
\end{subfigure}
\caption{Comparing Q-EP \eqref{fig:cls_QEP_logits} and deep Q-EP \eqref{fig:cls_DQEP_logits} with GP \eqref{fig:cls_GP_logits}, deep GP \eqref{fig:cls_DGP_logits}, DKL-GP \eqref{fig:cls_DKLGP_logits} and DSPP \eqref{fig:cls_DSPP_logits} on a classification problem defined on annular rhombus.}
\label{fig:cls_logits}
\end{figure}

\end{document}